\documentclass[dvipsnames]{article}

\PassOptionsToPackage{numbers}{natbib}

\usepackage[final]{neurips_2020}

\usepackage[utf8]{inputenc} % allow utf-8 input
\usepackage[T1]{fontenc}    % use 8-bit T1 fonts
\usepackage{hyperref}       % hyperlinks
\usepackage{url}            % simple URL typesetting
\usepackage{booktabs}       % professional-quality tables
\usepackage{amsfonts}       % blackboard math symbols
\usepackage{nicefrac}       % compact symbols for 1/2, etc.
\usepackage{microtype}      % microtypography
\usepackage{graphicx}
\usepackage{subfigure}
\usepackage{booktabs} 

\usepackage{amsmath,amssymb}
\usepackage{amsthm}
\usepackage{xcolor}
\usepackage{wrapfig}

\DeclareMathOperator*{\argmin}{\arg\!\min}
\DeclareMathOperator*{\argmax}{\arg\!\max}
\definecolor{todogreen}{rgb}{0,0.6,0}

\newtheorem{theorem}{Theorem}
\newtheorem{proposition}{Proposition}
\newtheorem{corollary}{Corollary}

\newcommand{\wname}{Optimistic }

\newcommand{\optimOper}{\ensuremath{\mathcal{T}^*}}
\newcommand{\qmixOper}{\ensuremath{\mathcal{T}^*_{\text{Qmix}}}}
\newcommand{\wOper}{\ensuremath{\mathcal{T}^*_{w}}}
\newcommand{\wqmixOper}{\ensuremath{\mathcal{T}^*_{\text{WQMIX}}}}

\newcommand{\qmixSpace}{\ensuremath{\mathcal{Q}^{mix}}}

\newcommand{\qmixProj}{\ensuremath{\Pi_{\text{Qmix}}}}
\newcommand{\wProj}{\ensuremath{\Pi_{w}}}

\newcommand{\qCentral}{\ensuremath{\hat{Q}^*}}
\newcommand{\qtot}{\ensuremath{Q_{tot}}}

\newcommand{\approxU}{\ensuremath{\mathbf{\hat{u}}}}
\newcommand{\U}{\ensuremath{\mathbf{u}}}

\title{Weighted QMIX: Expanding Monotonic Value Function Factorisation for
            Deep Multi-Agent Reinforcement Learning}

\author{
  Tabish Rashid, ~Gregory Farquhar\thanks{Now at Google DeepMind.}~, ~Bei Peng, ~Shimon Whiteson\\
  Department of Computer Science\\ 
  University of Oxford\\
  \texttt{\{tabish.rashid, gregory.farquhar, bei.peng, shimon.whiteson\}@cs.ox.ac.uk}\\
}

\begin{document}

\maketitle

\begin{abstract}
    QMIX is a popular $Q$-learning algorithm for cooperative MARL in the centralised training and decentralised execution paradigm.
In order to enable easy decentralisation, QMIX restricts the joint action $Q$-values it can represent to be a monotonic \textit{mixing} of each agent's utilities.
However, this restriction prevents it from representing value functions in which an agent's ordering over its actions can depend on other agents' actions.
To analyse this representational limitation, we first formalise the objective QMIX optimises, which allows us to view QMIX as an operator 
that first computes the $Q$-learning targets and then projects them into the space representable by QMIX.
This projection returns a representable $Q$-value that minimises the unweighted squared error across all joint actions.
We show in particular that this projection can fail to recover the optimal policy even with access to $Q^*$,
which primarily stems from the equal weighting placed on each joint action.
We rectify this by introducing a weighting into the projection, in order to place more importance on the better joint actions.
We propose two weighting schemes and prove that they recover the correct maximal action for any joint action $Q$-values, and therefore for $Q^*$ as well.
Based on our analysis and results in the tabular setting, we introduce two scalable versions of our algorithm,
Centrally-Weighted (CW) QMIX and Optimistically-Weighted (OW) QMIX and demonstrate improved performance on both predator-prey and challenging multi-agent StarCraft benchmark tasks \citep{samvelyan19smac}.

\end{abstract}

\section{Introduction}

Many critical tasks involve multiple agents acting in the same environment.
To learn good behaviours in such problems from agents' experiences, 
we may turn to multi-agent reinforcement learning (MARL).
Fully decentralised policies are often used in MARL, due to practical 
communication constraints or as a way to deal with an intractably large joint 
action space.
However, when training in simulation or under controlled conditions we may 
have access to additional information, and agents can freely share their
observations and internal states.
Exploiting these possibilities can greatly improve the efficiency of learning
\citep{foerster2016learning,foerster_counterfactual_2017}.

In this paradigm of centralised training for decentralised execution,
QMIX \cite{rashid2018qmix} is a popular $Q$-learning algorithm with 
state-of-the-art 
performance on the StarCraft Multi-Agent Challenge \citep{samvelyan19smac}. 
QMIX represents the optimal joint action value function using a monotonic 
\emph{mixing} 
function of per-agent utilities.
This restricted function class \qmixSpace~allows for efficient 
maximisation during training, and easy decentralisation of the learned policy.
However, QMIX is unable to represent joint action value functions that are 
characterised as \textit{nonmonotonic} \citep{mahajan2019maven}, i.e.,
an agent's ordering over its own 
actions depends on other agents' actions.
Consequently, QMIX cannot solve tasks that require 
significant coordination within a given timestep \citep{son2019qtran,bhmer2019deep}.
In this work, we analyse an idealised, tabular version of QMIX to study this
representational limitation, and then develop algorithms 
to resolve these limitations in theory and in practice.

We formalise the objective that QMIX optimises, which allows us to view QMIX as 
an operator that first computes the $Q$-learning targets and then projects them 
into \qmixSpace~ by 
minimising the unweighted squared error across all joint actions.
We show that, since in general $Q^* \notin \qmixSpace$, the 
projection of $Q^*$, which we refer to as \qtot, can have incorrect estimates for the optimal joint action, yielding suboptimal policies.
These are fundamental limitations of the QMIX algorithm independent from 
exploration and compute constraints, and occur even with access to the true 
$Q^*$.

These limitations primarily arise because QMIX's projection of $Q^*$ yields a 
$Q_{tot}$ that places equal importance on approximating the $Q$-values for 
\emph{all} joint actions. 
Our key insight is that if we ultimately care only about 
the greedy optimal policy, it is more important to accurately represent 
the value of the optimal joint action than the suboptimal ones.  
Therefore, we can improve the policy recovered from $Q_{tot}$ by appropriately weighting each joint action when projecting $Q^*$ into \qmixSpace.

Based on this intuition, we introduce a weighting function into our projection.
In the idealised tabular setting we propose two weighting functions and prove that the projected \qtot~recovers the correct maximal action for any $Q$, and therefore for $Q^*$ as well. 
Since this projection always recovers the correct maximal joint action, 
we benefit from access to $Q^*$ (or a learned approximation of it).
To this end, we introduce a learned approximation of $Q^*$, from an unrestricted function class, which we call \qCentral.
By using \qtot, now a weighted projection of \qCentral, to perform maximisation, we show that \qCentral~ converges to $Q^*$ and that \qtot~thus recovers the optimal policy.

Based on our analysis and results in the tabular setting, we present two scalable versions of our algorithm, Centrally-Weighted (CW) QMIX 
and Optimistically-Weighted (OW) QMIX.
We demonstrate their improved ability to cope 
with environments with nonmonotonic value functions,
by showing superior performance in a predator-prey task in which
the trade-offs made by QMIX prevent it from solving the task. 
Additionally, we demonstrate improved robustness over QMIX to the amount of 
exploration performed, by showing better empirical performance on a range of 
SMAC maps.
Our ablations and additional analysis experiments demonstrate the importance of both a weighting and an unrestricted \qCentral~ in our algorithm.

\section{Background}
\label{sec:background}

A fully cooperative multi-agent sequential decision-making task can be described as a \emph{decentralised partially observable Markov decision process} (Dec-POMDP) \citep{oliehoek2016concise} consisting of a tuple $G=\left\langle S,U,P,r,Z,O,n,\gamma\right\rangle$. 
$s \in S$ describes the true state of the environment.
At each time step, each agent $a \in A \equiv \{1,...,n\}$ chooses an action $u_a\in U$, forming a joint action $\mathbf{u}\in\mathbf{U}\equiv U^n$. 
This causes a transition on the environment according to the state transition function $P(s'|s,\mathbf{u}):S\times\mathbf{U}\times S\rightarrow [0,1]$. 
All agents share the same reward function $r(s,\mathbf{u}):S\times\mathbf{U}\rightarrow\mathbb{R}$ and $\gamma\in[0,1)$ is a discount factor. 

Due to the \textit{partial observability}, each agent's individual observations $z\in Z$ are produced by the observation function $O(s,a):S\times A\rightarrow Z$. 
Each agent has an action-observation history $\tau_a\in T\equiv(Z\times U)^*$, on which it conditions a (potentially stochastic) policy $\pi^a(u_a|\tau_a):T\times U\rightarrow [0,1]$. 
$\boldsymbol{\tau}$ denotes the action-observation histories of all agents (up to that timestep).
The joint policy $\pi$ has a joint \textit{action-value function}: $Q^\pi(s_t, \mathbf{u}_t)=\mathbb{E}_{s_{t+1:\infty}, \mathbf{u}_{t+1:\infty}} \left[R_t|s_t,\mathbf{u}_t\right]$, where $R_t=\sum^{\infty}_{i=0}\gamma^ir_{t+i}$ is the \textit{discounted return}.
For our idealised tabular setting, we consider a \textit{fully observable} setting in which each agent's observations are the full state. This is equivalent to a multi-agent MDP (MMDP) \citep{oliehoek2016concise} which is itself equivalent to a standard MDP with $U^n$ as the action space.

We adopt the \emph{centralised training and decentralised execution} paradigm \citep{oliehoek2008optimal,kraemer2016multi}.
During training our learning algorithm has access to the true state $s$ and 
every agent's action-observation history, as well as the freedom to share all 
information between agents.
However, during testing (execution), each agent has access only to its own action-observation history.

\subsection{QMIX}
\label{sec:background_qmix}
The Bellman optimality operator is defined by:
\begin{equation}
\optimOper Q(s,\mathbf{u}) := \mathbb{E} [r + \gamma \max_{\mathbf{u'}} Q(s',\mathbf{u'})]
,
\label{eq:bellman_opt}
\end{equation}
where the expectation is over the next state $s' \sim P(\cdot|s, \mathbf{u})$ 
and reward $r \sim r(\cdot|s,\mathbf{u})$.
Algorithms based on $Q$-learning \citep{watkins1992q} use samples from the 
environment to estimate the expectation in \eqref{eq:bellman_opt}, in order to 
update their estimates of $Q^*$.

VDN \citep{sunehag2018value} and QMIX are $Q$-learning algorithms for the 
cooperative MARL setting, which estimate the optimal joint action value 
function $Q^*$ as \qtot, with specific forms.
VDN factorises \qtot~into a sum of the per-agent utilities: $\qtot(s,\U) = 
\sum_{a=1}^n Q_a(s,u_a)$, whereas QMIX combines the
per-agent utilities via a continuous monotonic function that is state-dependent: $f_s(Q_1(s,u_1),...,Q_n(s,u_n)) = Q_{tot}(s,\mathbf{u})$, where $\frac{\partial f_s}{\partial Q_a}  \geq 0,~ \forall a \in A \equiv \{1,...,n\}$.

In the deep RL setting neural networks with parameters $\theta$ are used as 
function approximators, and
QMIX is trained much like a DQN \citep{mnih2015human}.
Considering only the fully-observable setting for ease of presentation, a 
replay buffer stores transition tuples $(s,\U, r, s', d)$,
in which the agents take joint action $\U$ in state $s$, receive reward $r$ and 
transition to $s'$ and $d$ is a boolean indicating if $s'$ is a terminal state.
QMIX is trained to minimise the squared TD error on a minibatch of 
$b$ samples from the replay buffer:
$
\sum_{i=1}^b (\qtot(s, \U;\theta) - y_i)^2,
$
where $y_i = r + \gamma \max_{\U'} \qtot(s', \U';\theta^-)$ are the targets, and $\theta^-$ are the parameters of a \textit{target network} that are periodically copied from $\theta$. 
The monotonic \textit{mixing} function $f_s$ is parametrised as a feedforward network, whose non-negative weights are generated by hypernetworks \citep{ha2017hyper} that take the state as input.
\section{QMIX Operator}
\label{sec:tabular_qmix}

In this section we examine an operator that represents an \textit{idealised} version of QMIX in a tabular setting.
The purpose of our analysis is primarily to understand the fundamental 
limitations of QMIX that stem from its training objective and the restricted 
function class it uses. We write this function class as \qmixSpace~:
\begin{align*}
\qmixSpace := \{Q_{tot} | Q_{tot}(s,\mathbf{u}) = 
f_s(Q_1(s,u_1),...Q_n(s,u_n)), %\\&
\frac{\partial f_s}{\partial Q_a}  \geq 0, Q_a(s,u) \in \mathbb{R}\}.
\end{align*}
This is the space of all $Q_{tot}$ that can be represented by monotonic 
funtions of tabular $Q_a(s,u)$. At each iteration of our idealised algorithm, we constrain $Q_{tot}$ to 
lie in \qmixSpace~by solving the following optimisation problem:
\begin{equation}
\argmin_{q \in \qmixSpace} \sum_{\U \in \mathbf{U}} (\optimOper \qtot(s,\U) - q 
(s,\U))^2.
\label{eq:qmix_objective}
\end{equation}
To avoid any confounding factors regarding exploration, we assume this 
optimisation is performed for all states and joint actions at each iteration, 
as in planning algorithms like value iteration \citep{puterman2014markov}.
We also assume the optimisation is performed exactly.
However, since it is not guaranteed to have a unique solution, a random $q$ is 
returned from the set of objective-minimising candidates.

By contrast, the original QMIX algorithm in the deep RL setting interleaves 
exploration and approximate optimisation of this objective, using samples from 
the environment to approximate \optimOper, and a finite number of gradient 
descent steps to estimate the argmin.
Additionally, sampling uniformly from a replay buffer of the most recent 
experiences does not strictly lead to a uniform weighting across joint actions. 
Instead the weighting is proportional to the frequency at which a joint action 
was taken.
Incorporating this into our analysis would introduce significant complexity and detract from our main focus: to analyse the limitations of restricting the representable function class to \qmixSpace.

The optimisation in \eqref{eq:qmix_objective} can be separated into 
two distinct parts: the first computes targets using 
\optimOper, and the second projects those targets into \qmixSpace.
We define the corresponding projection operator \qmixProj~as follows:
\begin{equation*}
\qmixProj Q := \argmin_{q \in \qmixSpace} \sum_{\mathbf{u} \in \mathbf{U}} 
(Q(s,\mathbf{u}) - q(s,\mathbf{u}))^2
\end{equation*}
We can then define $\qmixOper := \qmixProj \optimOper$ as the QMIX operator, 
which exactly corresponds to the objective in \eqref{eq:qmix_objective}.

There is a significant literature studying projections of value functions into 
linear function spaces for RL (see \citep{sutton2018reinforcement} for a detailed introduction and overview).
However, despite some superficial similarities, our focus is considerably 
different in several ways.
First, we consider a nonlinear projection, rendering many methods of analysis 
inapplicable.
Second, as the $Q_a(s,u)$ are tabular, there is no tradeoff in the quality of 
representation across different states, and consequently no need to weight the 
different states in the projection.
By contrast, unlike linear $Q$-functions, our restricted space does induce 
tradeoffs in the quality of represenation across different joint actions, and 
weighting them differently in the projection is central to our method.
Third, the targets in our optimisation are fixed to $\optimOper \qtot(s,\U)$ 
at each iteration, rather than depending on $q$ as they would in a minimisation 
of Mean Squared Bellman Error (MSBE) or Mean Squared Projected Bellman Error (MSPBE) in linear RL.
This makes our setting closer to that of fitted $Q$-iteration 
\cite{ernst2005tree} in which a regression problem is solved at each iteration 
to fit a function approximator from a restricted class to $Q$-learning-style 
targets.
Our focus of study is the unique properties of the particular function 
space \qmixSpace, and the tradeoffs in representation quality induced by 
projection into it.

\subsection{Properties of $\qmixOper$}
\label{sec:qmix_properties}

To highlight the pitfalls of the projection \qmixProj~into \qmixSpace, we 
consider the effect of applying \qmixProj~to the true $Q^*$, which is readily 
available in deterministic normal form games where $Q^*$ is just the immediate 
reward.

\setlength\intextsep{0pt}
\begin{wraptable}{r}{0pt}
	\begin{tabular}{ | c | c | }
		\hline
		1 & 0 \\
		\hline
		0 & 1  \\
		\hline
	\end{tabular}
	~~
	\begin{tabular}{ | c | c | }
		\hline
		1 & 1/3 \\
		\hline
		1/3 & 1/3  \\
		\hline
	\end{tabular}
	~~
	\begin{tabular}{ | c | c | }
		\hline
		1/3 & 1/3 \\
		\hline
		1/3 & 1  \\
		\hline
	\end{tabular}
	\hfill
	\caption{Non-monotonic payoff matrix (Left) and the two possible \qtot's 
	returned by \qmixProj~(Middle and Right).}
	\label{matrix:non_monotonic}
\end{wraptable}

\textbf{\qmixOper~is not a contraction.}
The payoff matrix in Table \ref{matrix:non_monotonic} (Left) is a simple 
example of a value function that cannot be perfectly represented in \qmixSpace. 
Table \ref{matrix:non_monotonic} (Middle) and (Right) show two distinct 
$Q_{tot}$, both of which are global minima of the optimisation solved by 
$\qmixProj Q^*$.
Hence, \qmixOper~is not a contraction, which would have a unique fixed 
point.

\textbf{QMIX's argmax is not always correct.}
There exist $Q$-functions such that $\argmax \qmixProj Q \neq \argmax~Q$.
For example, the payoff matrix in Table \ref{matrix:argmax_wrong} (Left) (from 
\citet{son2019qtran}) produces a value function for which QMIX's approximation 
(Right) does not result in the correct argmax.
 
\setlength\intextsep{0pt}
\begin{wraptable}{r}{0pt}
	\begin{tabular}{ | c | c | c |}
		\hline
		8 & -12 & -12 \\
		\hline
		-12 & 0 & 0 \\
		\hline
		-12 & 0 & 0 \\
		\hline
	\end{tabular}
	~~
	\begin{tabular}{ | c | c | c |}
		\hline
		-12 & -12 & -12 \\
		\hline
		-12 & 0 & 0 \\
		\hline
		-12 & 0 & 0 \\
		\hline
	\end{tabular}
	\hfill
	\caption{Payoff matrix (Left) in which \qtot~returned from \qmixProj~ has 
	an incorrect argmax (Right).}
	\label{matrix:argmax_wrong}
\end{wraptable}

\textbf{QMIX can underestimate the value of the optimal joint action.}
Furthermore, if it has an incorrect argmax, the value of the true optimal 
joint action can be underestimated, e.g., $-12$ instead of $8$ in Table 
\ref{matrix:argmax_wrong}.
If QMIX gets the correct argmax then it represents the maximum $Q$-value 
perfectly (proved formally in Appendix \ref{app:proofs}).
However, if QMIX's argmax joint action is not the true optimal joint action 
then QMIX can underestimate the value of that action.

These failure modes are problematic because they show fundamental limitations
of QMIX, that are independent from: 1) compute constraints, since we exactly 
minimise the objectives posed; 2) exploration, since we update every 
state-action pair; and 
3) parametrisation of the mixing function and agent utilities, since we are assume 
that \qtot~can be any member of \qmixSpace, whereas in practice we can only 
represent a subset of \qmixSpace.
\section{Weighted QMIX Operator}
\label{sec:weighted_qmix}

In this section we introduce an operator for an \textit{idealised} version of our algorithm, \emph{Weighted QMIX} (WQMIX), 
in order to compare it to the operator we introduced for QMIX.

The negative results in Section \ref{sec:qmix_properties} concern 
the scenario in which we optimise QMIX's loss function across \emph{all} 
joint actions for every state.
We argue that this equal weighting over joint actions when performing the 
optimisation in \eqref{eq:qmix_objective} is responsible for the possibly 
incorrect argmax of the objective-minimising solution.
Consider the example in Table \ref{matrix:argmax_wrong}.
A monotonic $Q_{tot} \in \qmixSpace$ cannot increase its estimate of the value 
of the single 
optimal joint action above -12 without either increasing the estimates of the 
value of the bad joint actions above their true value of -12, or decreasing the 
estimates of the zero-value joint actions below -12.
The error for misestimating several of the suboptimal joint actions would 
outweigh the improvement from better estimating the single optimal joint 
action.
As a result the optimal action value is underestimated and the resulting policy 
is suboptimal.

By contrast, consider the extreme case in which we only optimise the 
loss for the \emph{single} optimal joint action $\mathbf{u^*}$.
For a single action, the representational limitation of QMIX has no effect so 
we can optimise the objective perfectly, recovering the value of the optimal 
joint action.

However, we still need to learn that the other action values are lower than 
$Q_{tot}(\mathbf{u^*})$ in order to recover the optimal policy.
To prioritise estimating $Q_{tot}(\mathbf{u^*})$ well, while still anchoring 
down the value estimates for other joint actions, we can add a suitable 
weighting function $w$ into the projection operator of QMIX:
\begin{equation}
\wProj Q := \argmin_{q \in \qmixSpace} \sum_{\mathbf{u} \in \mathbf{U}} w(s,\mathbf{u})(Q(s,\mathbf{u}) - q(s,\mathbf{u}))^2.
\label{eq:weighted_q_proj}
\end{equation}
The weighting function $w : S \times \mathbf{U} \rightarrow (0,1]$ weights the importance of each joint action in QMIX's loss function. 
It can depend on more than just the state and 
joint action, but we omit this from the notation for simplicity.
Setting $w(s, \U) = 1$ recovers the projection operator \qmixProj.

\subsection{Weightings}

The choice of weighting is crucial to ensure that WQMIX can overcome the 
limitations of QMIX. 
As shown in Section \ref{sec:tabular_qmix}, even if we have access to $Q^*$, if 
we use a uniform weighting then we can still end up with the wrong argmax after 
projection into the monotonic function space.
We now consider two different weightings and
show in Theorems \ref{th:critic_w} and \ref{th:hyst_w} that these choices of $w$ ensure that the $Q_{tot}$ returned from the projection has the correct argmax. 
The proofs of these theorems can be found in Appendix \ref{app:proofs}.
For both weightings, let $\alpha \in (0,1]$ and consider the weighted projection of an arbitrary joint action value function $Q$.

\paragraph{Idealised Central Weighting} The first weighting, which we call Idealised Central Weighting, is quite simple:
\begin{equation}
    w(s,\mathbf{u}) = 
    \begin{cases}
    1 & \mathbf{u} = \mathbf{u}^* = \argmax_{\U} Q(s,\U)\\
    \alpha & \text{otherwise}.
    \end{cases}
    \label{eq:critic_w}
\end{equation}
To ensure that the weighted projection returns a $Q_{tot}$ with the correct argmax, we simply down-weight every suboptimal action.
However, this weighting requires computing the maximum across the joint action space, which is often infeasible.
In Section \ref{sec:deeprl} we discuss an approximation to this weighting in the deep RL setting.

\begin{theorem}
Let $w$ be the Idealised Central Weighting from \eqref{eq:critic_w}.
Then $\exists \alpha>0$ such that $\argmax \Pi_w Q = \argmax Q$ for any $Q$. 
\label{th:critic_w}
\end{theorem}
Theorem \ref{th:critic_w} provides a sanity check that this choice of weighting 
guarantees we recover a \qtot~with the correct argmax in this idealised 
setting, with a nonzero weighting for suboptimal actions.

\paragraph{\wname Weighting} The second weighting, which we call Optimistic Weighting, affords a practical implementation:
\begin{equation}
    w(s,\mathbf{u}) = 
    \begin{cases}
    1 & \qtot(s,\U) < Q(s,\U) \\
    \alpha & \text{otherwise}.
    \end{cases}
    \label{eq:hyst_w}
\end{equation}
This weighting assigns a higher weighting to those joint actions that are underestimated relative to $Q$,
and hence could be the true optimal actions (in an optimistic outlook).

\begin{theorem}
Let $w$ be the \wname Weighting from \eqref{eq:hyst_w}.
Then $\exists \alpha>0$ such that, $\argmax \Pi_w Q = \argmax Q$ for any $Q$.
\label{th:hyst_w}
\end{theorem}

Theorem \ref{th:hyst_w} shows that \wname Weighting also recovers a \qtot~with 
the correct argmax.

\subsection{Weighted QMIX Operators}

We have shown that these two weightings are guaranteed to recover the correct maximum joint action for any $Q$, and therefore for $Q^*$ as well.
This is in contrast to the uniform weighting ($w=1$) of QMIX, which can fail to 
recover the correct optimal joint action even for simple matrix games.
The weighted projection now allows us to fully take advantage of $Q^*$.

Since we do not have access to the true optimal value function in general, we 
learn an approximation to it: \qCentral, which does not need to lie in the 
restricted monotonic function space \qmixSpace.
Performing an exact maximisation of \qCentral~requires a search over the entire 
joint action space, which is typically intractable and does not admit 
decentralisation.
We instead use our QMIX approximation \qtot~to \textit{suggest} the maximum 
joint action(s), which can then be evaluated by \qCentral.

Learning \qCentral~instead of using \qtot~in its place brings some advantages. 
First, it allows us a richer representational class to approximate $Q^*$ with,
 since we place no restrictions on the form of \qCentral.
In the idealised tabular setting, $Q^*$ is exactly representable by \qCentral.
Second, since we are weighting each joint action in \wProj, \qtot~(unlike 
\qCentral) likely has less accurate estimates for those joint actions 
with a low weighting.
Due to these factors, we may bootstrap using more accurate estimates by 
using \qCentral~instead of \qtot.
These properties are necessary to ensure that WQMIX converges to the 
optimal policy.
The operator used to update \qCentral~is:
\begin{equation}
\wOper \qCentral (s,\mathbf{u}) := \mathbb{E} [r + \gamma \qCentral(s',\argmax_{\U'} \qtot(s',\U'))].
\label{eq:bellman_critic_opt}
\end{equation} 
Since $\qtot$ is monotonic, the argmax in \eqref{eq:bellman_critic_opt} is tractable. Similarly \qtot~is updated in tandem using:
\begin{equation}
\wqmixOper \qtot:= \wProj \wOper \qCentral
\end{equation}
\wOper~is similar to the Bellman Optimality Operator in \eqref{eq:bellman_opt} but does not directly maximise over \qCentral.
Instead it uses $\qtot \in \qmixSpace$
 to suggest a maximum joint action.
Setting $w$ to be uniform ($w = 1$) here does not recover QMIX since 
we are additionally learning \qCentral.

Finally, using our previous results, we show that \qCentral~converges to $Q^*$ and that \qtot~recovers an optimal policy.
This provides a firm theoretical foundation for Weighted QMIX: in an idealised setting it converges to the optimal policy, whereas QMIX does not.  

\begin{corollary}
\label{cor:fixed_point}
Letting $w$ be the Idealised Central or \wname Weighting, then
$\exists \alpha > 0$ such that
the unique fixed point of \wOper~is $Q^*$.
Furthermore, $\wProj Q^* \subseteq \qmixSpace$ recovers an optimal policy, and
$\max \wProj Q^*(s,\cdot) = \max Q^*(s,\cdot)$.
\end{corollary}

In this section we have shown that an idealised version of Weighted QMIX can converge to $Q^*$ and recover an optimal policy.
Restricting \qtot~to lie in \qmixSpace~does not 
prevent us from representing an optimal policy, since there is \emph{always} an 
optimal deterministic policy \citep{puterman2014markov} and all deterministic 
policies can be derived from the argmax of a $Q$ that lies in \qmixSpace.
Thus, we do \emph{not} expand the function class that we consider for 
\qtot.
Instead, we change the solution of the projection by introducing a 
weighting.

\section{Deep RL Algorithm}
\label{sec:deeprl}

So far, we have only considered an idealised setting in order to analyse the 
fundamental properties of QMIX and Weighted QMIX.
However, the ultimate goal of our analysis is to inform the development of new 
scalable RL algorithms, in combination with, e.g., neural network function 
approximators.
We now describe the realisation of Weighted QMIX for deep RL,
in a Dec-POMDP setting in which each agent does not observe the full state, as described in Section \ref{sec:background}.

There are three components to Weighted QMIX: 
1)  \qtot, i.e., the per-agent utilities $Q_a$ (from which the decentralised policies are derived) and the mixing network, 2) an unrestricted joint action \qCentral, and 3) a weighting function $w$, as in \wProj.

\vspace{-0.3cm}
\paragraph{$\mathbf{\qtot}$} 
The \qtot~component is largely the same as that of 
\citet{rashid2018qmix}, using the architecture from \citet{samvelyan19smac}.
$Q_{tot}$ is trained to minimise the following loss:
\begin{equation}
\sum_{i=1}^b w(s, \U)(\qtot(\boldsymbol{\tau}, \U, s) - y_i)^2,
\label{eq:qmix_loss}
\end{equation}
where $y_i := r + \gamma \qCentral(s', \boldsymbol{\tau'}, \argmax_{\U'} Q_{tot}(\boldsymbol{\tau'}, \U',s'))$ is treated as a fixed target. 
This differs from the idealised setting considered in Sections \ref{sec:tabular_qmix} and \ref{sec:weighted_qmix} because we are now only optimising the $Q$-values for the state-action pairs present in the minibatch sampled from the replay buffer, as opposed to every state-action pair.

\begin{figure}
    \centering
    \includegraphics[width=0.8\textwidth]{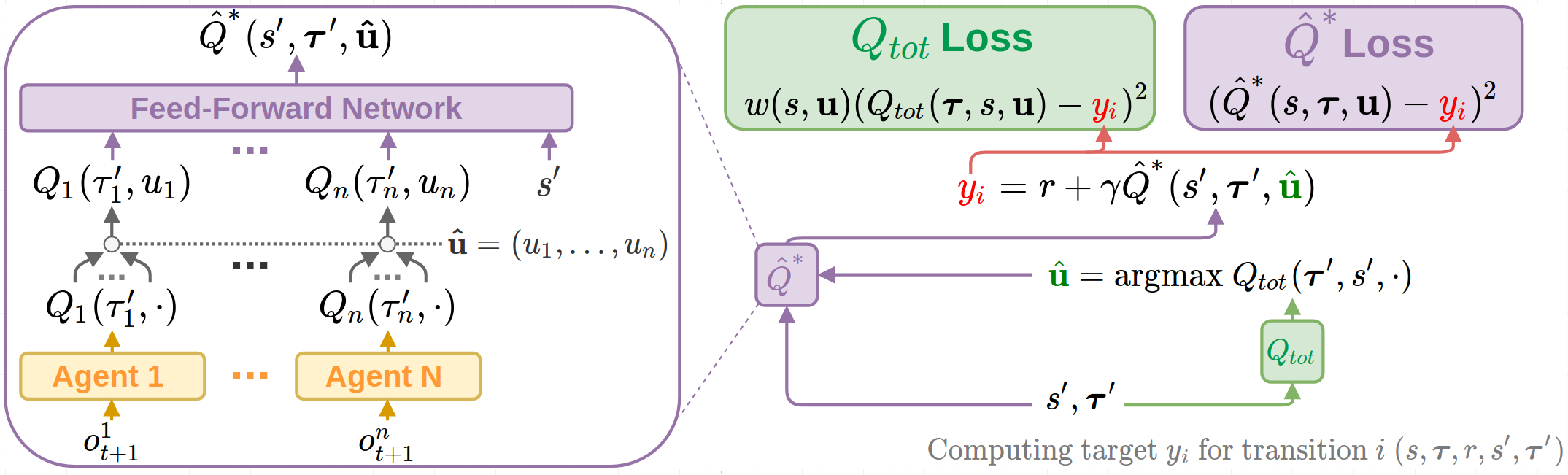}
    \caption{Deep RL Weighted QMIX setup. \textbf{Left:} The architecture used for \qCentral~. \textbf{Right:} How the targets $y_i$~ for each transition in the minibatch are computed and used.}
    \label{fig:wqmix_targets}
\end{figure}

\paragraph{Centralised \qCentral}
We use a similar architecture to \qtot~to represent \qCentral: agent networks 
whose chosen action's utility is fed into a mixing network.
\qCentral~is thus conditioned on the state $s$ and the agents' 
action-observation histories $\boldsymbol{\tau}$.
For the agent networks we use the same architecture as QMIX. They do not share parameters with the agents used in representing \qtot.
The mixing network for \qCentral~is a feed-forward network that takes the state 
and the appropriate actions' utilities as input. 
This mixing network is not constrained to be monotonic by using non-negative 
weights.
Consequently, we can simplify the architecture by having the state and agent 
utilities be inputs to \qCentral's mixing network, as opposed to having 
hypernetworks take the state as input and generate the weights.  
The architecture of \qCentral~is shown in Figure \ref{fig:wqmix_targets} (Left).
\qCentral~is trained to minimise the following loss, using $y_i$ from \eqref{eq:qmix_loss}: 
\begin{equation}
\sum_{i=1}^b (\qCentral(s, \boldsymbol{\tau}, \U) - y_i)^2.
\label{eq:qcentral_loss}
\end{equation}

\paragraph{Weighting Function}
Idealised Central Weighting requires knowing the maximal joint action over \qCentral, which is computationally infeasible.
In order to derive a practical algorithm, we must make approximations.
For each state-action pair in the sampled minibatch, 
the weighting we use is:
\begin{equation}
    w(s, \mathbf{u}) = 
    \begin{cases}
    1 & y_i > \qCentral(s, \boldsymbol{\tau}, \approxU^*) ~~\text{or}~~ \U = \approxU^* \\
    \alpha & \text{otherwise},
    \end{cases}
    \label{eq:deeprl_critic_w}
\end{equation}
where $\approxU^* = \argmax_{\U} Q_{tot}(\boldsymbol{\tau}, \U, s)$.
Since we do not know the maximal joint action for each state, we make a local approximation: 
if $y_i > \qCentral(s, \boldsymbol{\tau}, \approxU^*)$, then $\U$ might be the best joint action.
We use $\qCentral(s, \boldsymbol{\tau}, \approxU^*)$ instead of $\wOper \qCentral(s, \boldsymbol{\tau}, \approxU^*)$ since we do not have direct access to it.
We refer to this weighting function as \textbf{Centrally-Weighted QMIX (CW)}.

The \wname Weighting presented in \eqref{eq:hyst_w} does not require any approximations.
The exact weighting we use is:
\begin{equation*}
    w(s, \mathbf{u}) = 
    \begin{cases}
    1 & \qtot(\boldsymbol{\tau}, \U, s) < y_i \\
    \alpha & \text{otherwise}.
    \end{cases}
    \label{eq:deeprl_hyst_w}
\end{equation*}
We refer to it as \textbf{Optimistically-Weighted QMIX (OW)}.

In a deep RL setting, QMIX implicitly weights the joint-actions proportional to 
their execution by the behaviour policies used to fill the replay buffer.
This forces QMIX to make trade-offs in its $Q$-value approximation that are directly tied to the exploration strategy chosen.
However, as we have shown earlier, this can lead to poor estimates for the optimal joint action and thus yield suboptimal policies. 
Instead, Weighted QMIX separates the weighting of the joint actions from the 
behaviour policy.
This allows us to focus our monotonic approximation of $Q^*$ on the \textit{important} joint actions, thus encouraging better policies to be recovered irrespective of the exploration performed.

\section{Results}
\label{sec:results}

In this section we present our experimental results on the Predator Prey task considered by \citet{bhmer2019deep} and on a variety of SMAC\footnote{We utilise \texttt{SC2.4.6.2.69232} (the same version as \citep{samvelyan19smac}) instead of the newer \texttt{SC2.4.10}. Performance is \textbf{not} comparable across versions.} scenarios.
More details about the implementation of each are included in Appendix \ref{sec:setup}, as well as additional ablation experiments (Appendix \ref{app:results}).
For every graph we plot the median and shade the 25\%-75\% quartile.
Code is available at \url{https://github.com/oxwhirl/wqmix}.

\subsection{Predator Prey}

\begin{wrapfigure}{r}{0.50\textwidth}
    \centering
    \includegraphics[width=0.50\columnwidth]{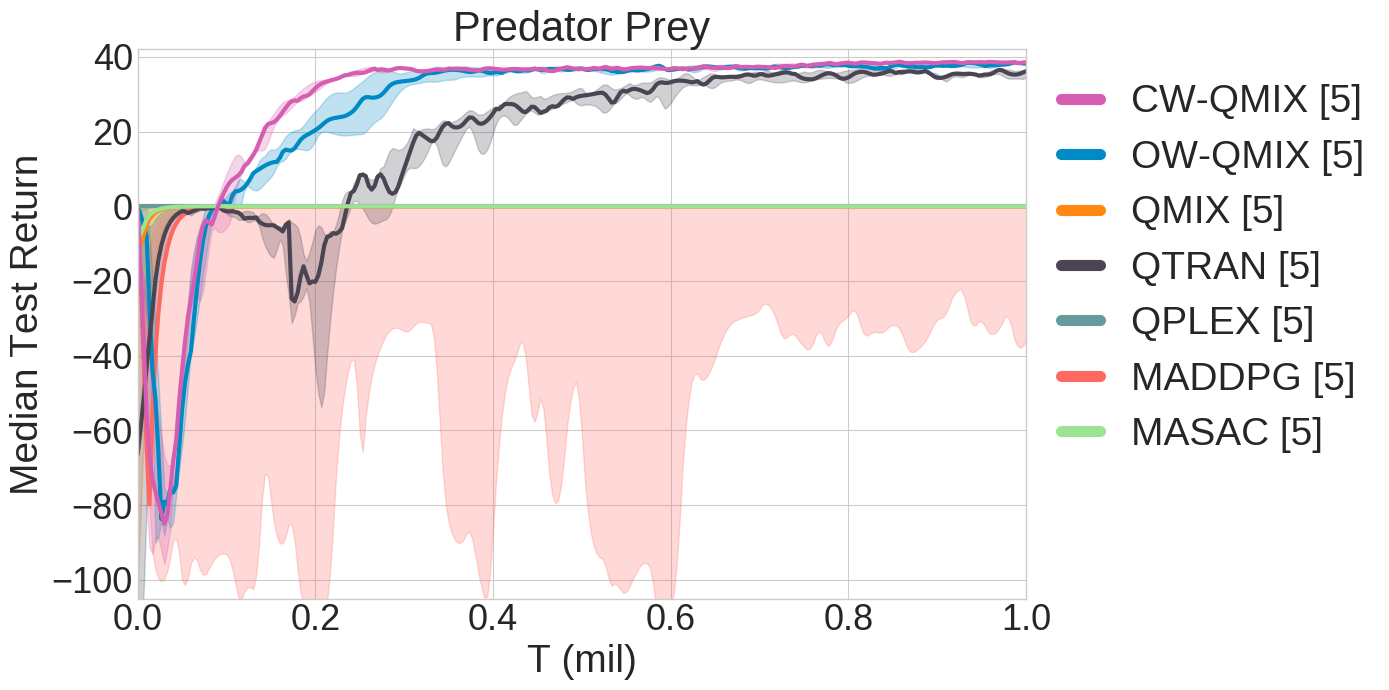}
    \caption{Median test return for the predator prey environment, comparing Weighted QMIX and baselines.}
    \label{graph:pred_prey}
\end{wrapfigure}
We first consider a partially-observable Predator Prey task involving 8 agents \citep{bhmer2019deep}, which 
was designed to test coordination between agents
by providing a punishment of -2 reward when only a single agent (as opposed to two agents) attempt to capture a prey.
Algorithms which suffer from \textit{relative overgeneralisation} \citep{panait2006biasing,wei2018multiagent}, or which make poor trade-offs in their representation (as VDN and QMIX do) can fail to solve this task.

As shown in \citep{bhmer2019deep}, QMIX fails to learn a policy that achieves positive test reward, and our results additionally show the same negative results for MADDPG and MASAC.
Interestingly, QPLEX also fails to solve the task despite not having any restrictions on the joint-action $Q$-values it can represent, suggesting difficulties in learning certain value functions.
Figure \ref{graph:pred_prey} shows that both CW-QMIX and OW-QMIX solve the task faster than QTRAN.

\subsection{SMAC}

\subsubsection{Robustness to increased exploration}
\begin{figure}[h]
    \centering
    \includegraphics[width=0.49\columnwidth]{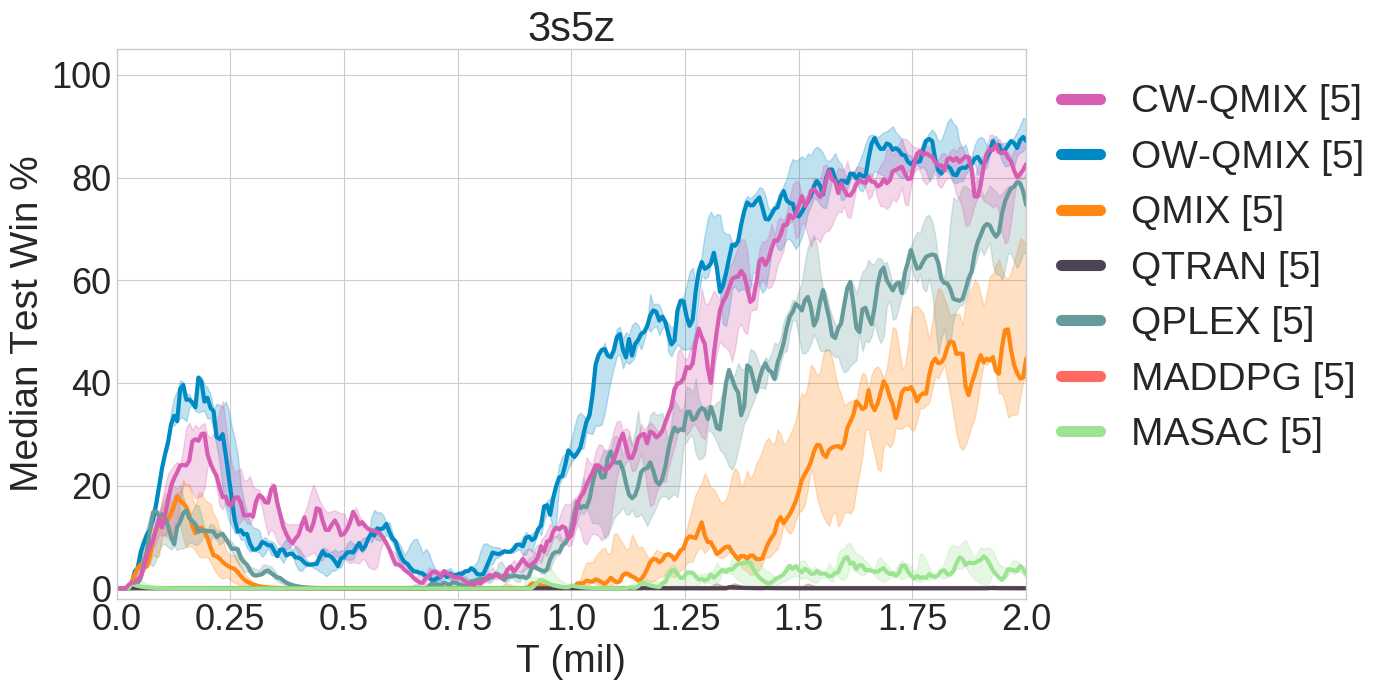}
    \includegraphics[width=0.49\columnwidth]{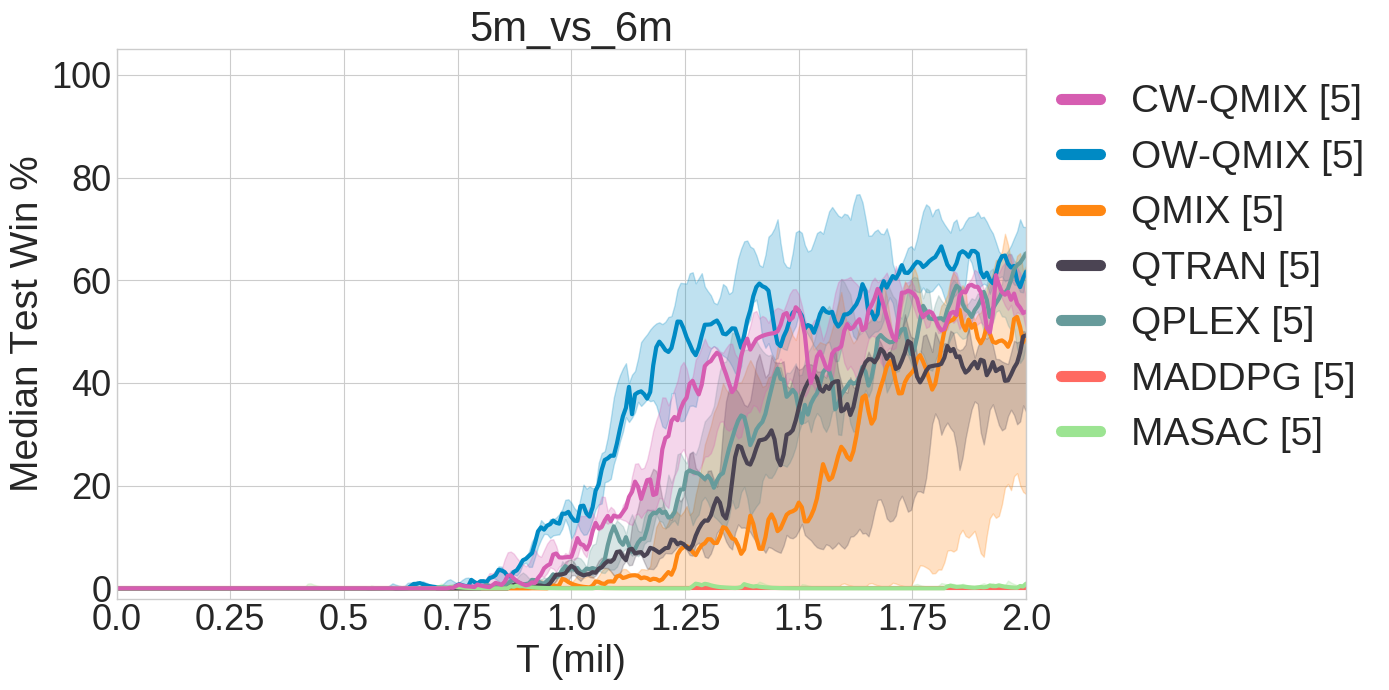}
    \caption{Median test win \% with an increased rate of exploration.}
    \label{graph:1mil_eps_smac}
\end{figure}
QMIX is particularly brittle when there is significant exploration being 
done, since it then tries (and fails) to represent the value of many suboptimal 
joint actions.
We would like our algorithms to efficiently learn from exploratory behaviour.
Hence, we use two of the easier SMAC maps, \textit{3s5z} and 
\textit{5m\_vs\_6m}, to test the robustness of our algorithms to exploration.
We use an $\epsilon$-greedy policy in which $\epsilon$ is 
annealed from $1$ to $0.05$ over 1 million timesteps, increased from 
the $50k$ used in \citep{samvelyan19smac}.
Figure \ref{graph:1mil_eps_smac} shows the results of these experiments, in which both Weighted QMIX variants significantly outperform all baselines.
Figure \ref{graph:6h8z_bvb} (Left) additionally compares the performance of QMIX and Weighted QMIX on \texttt{bane\_vs\_bane}, a task with 24 agents, across two $\epsilon$-schedules. We can see that both variants of Weighted QMIX are able to solve the task irrespective of the level of exploration, whereas QMIX fails to do so. 

\subsubsection{Necessity of increased exploration}

\begin{figure}[h]
    \centering
    \includegraphics[width=0.49\columnwidth]{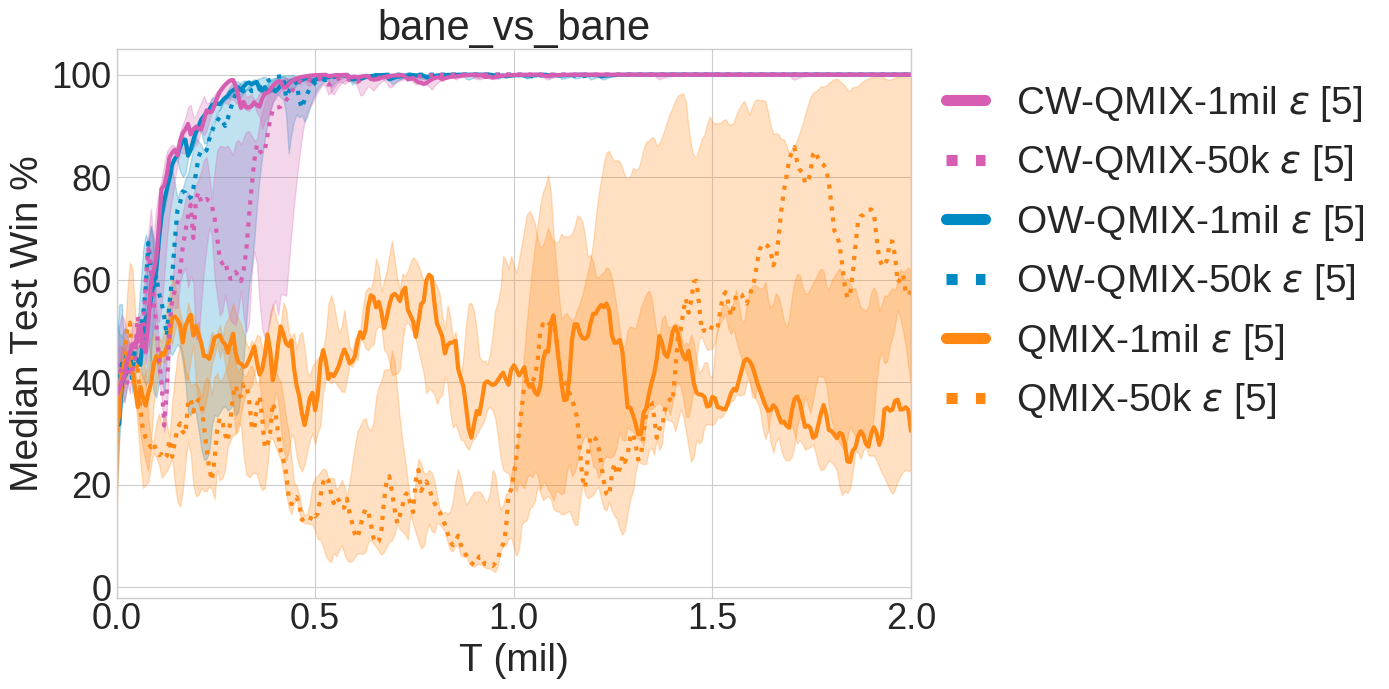}
    \includegraphics[width=0.49\columnwidth]{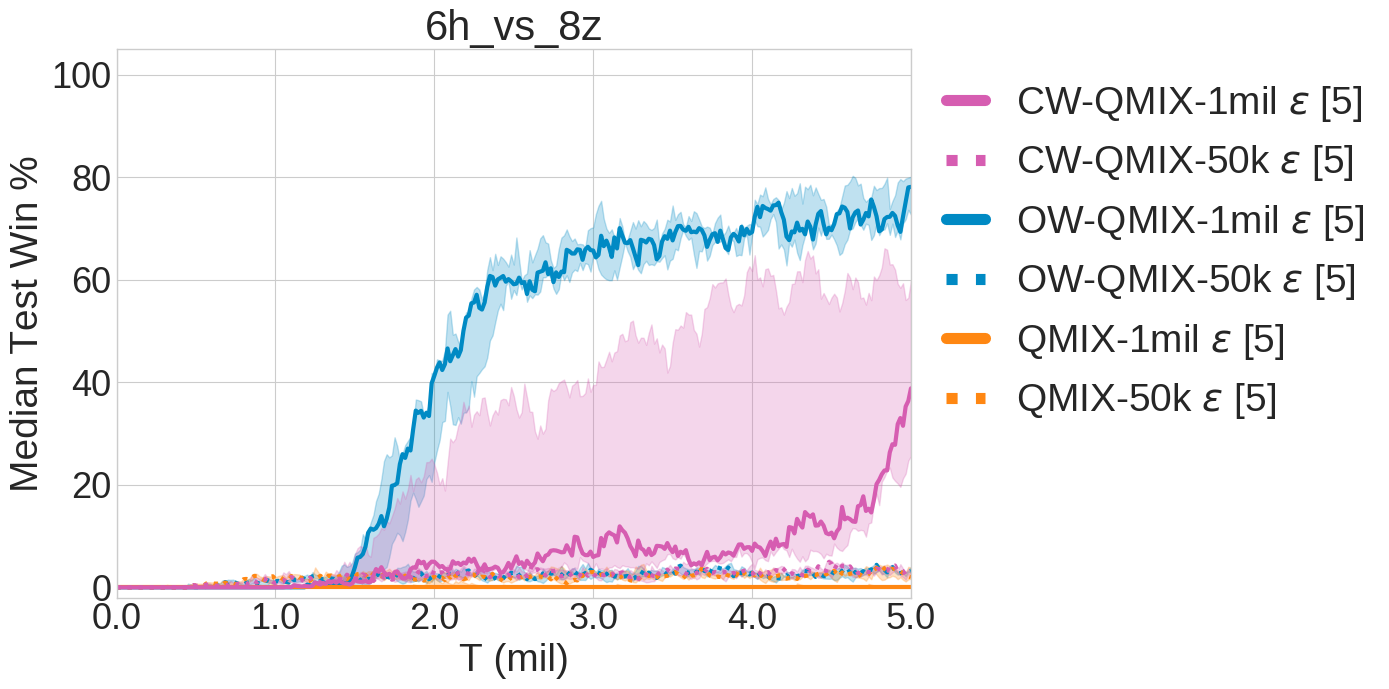}
    \caption{Median test win \% with 2 exploration schedules on \texttt{bane\_vs\_bane} and \textit{6h\_vs\_8z}.}
    \vspace{-0.5cm}
    \label{graph:6h8z_bvb}
\end{figure}

Next, we compare our method on the challenging \textit{6h\_vs\_8z}, which is classified as a \textit{super hard} SMAC map due to current method's poor performance \citep{samvelyan19smac}.
Figure \ref{graph:6h8z_bvb} (Right) compares QMIX and Weighted QMIX with two differing exploration schedules (annealing 
$\epsilon$ over 50k or 1 million timesteps, denoted -50k $\epsilon$ and -1mil 
$\epsilon$ respectively).
We can see that a larger rate of exploration is required, and only Weighted QMIX can successfully recover a good policy, demonstrating the benefits of our method for improving performance in a challenging coordination problem.

\subsubsection{Limitations}
\begin{figure}[h]
    \centering
    \includegraphics[width=0.49\columnwidth]{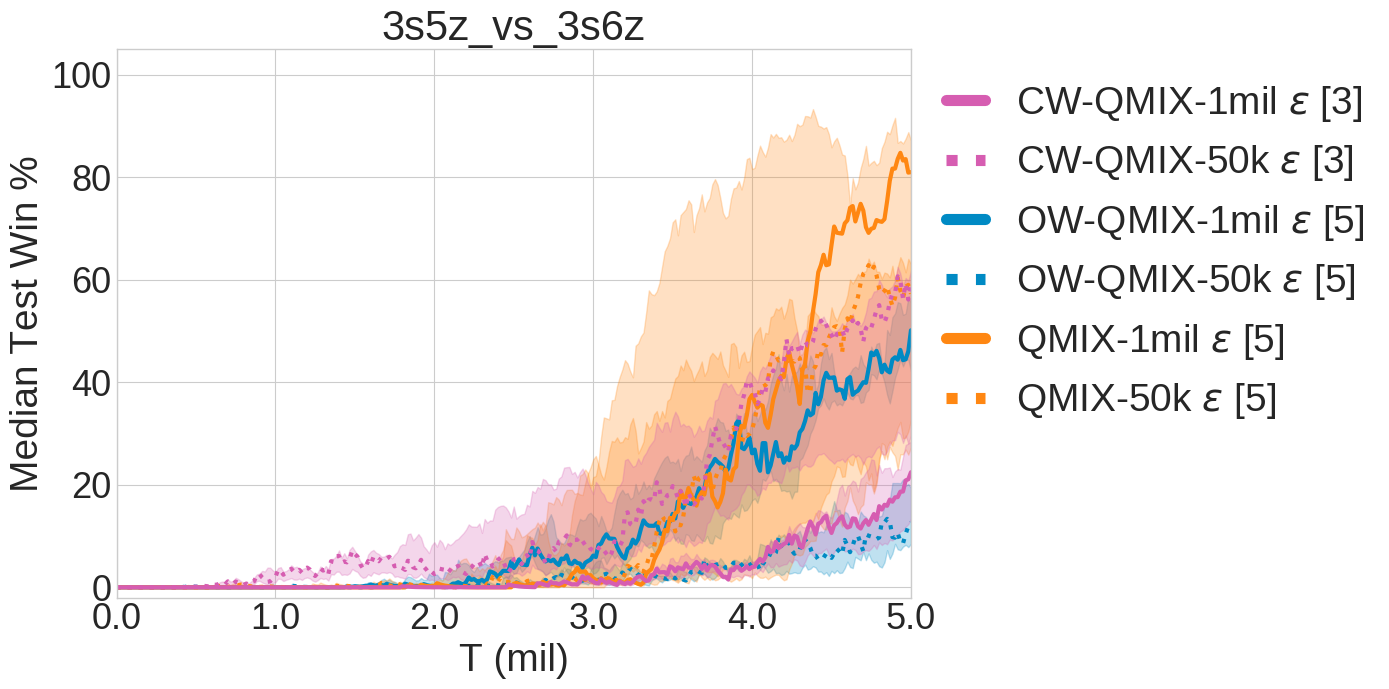}
    \includegraphics[width=0.49\columnwidth]{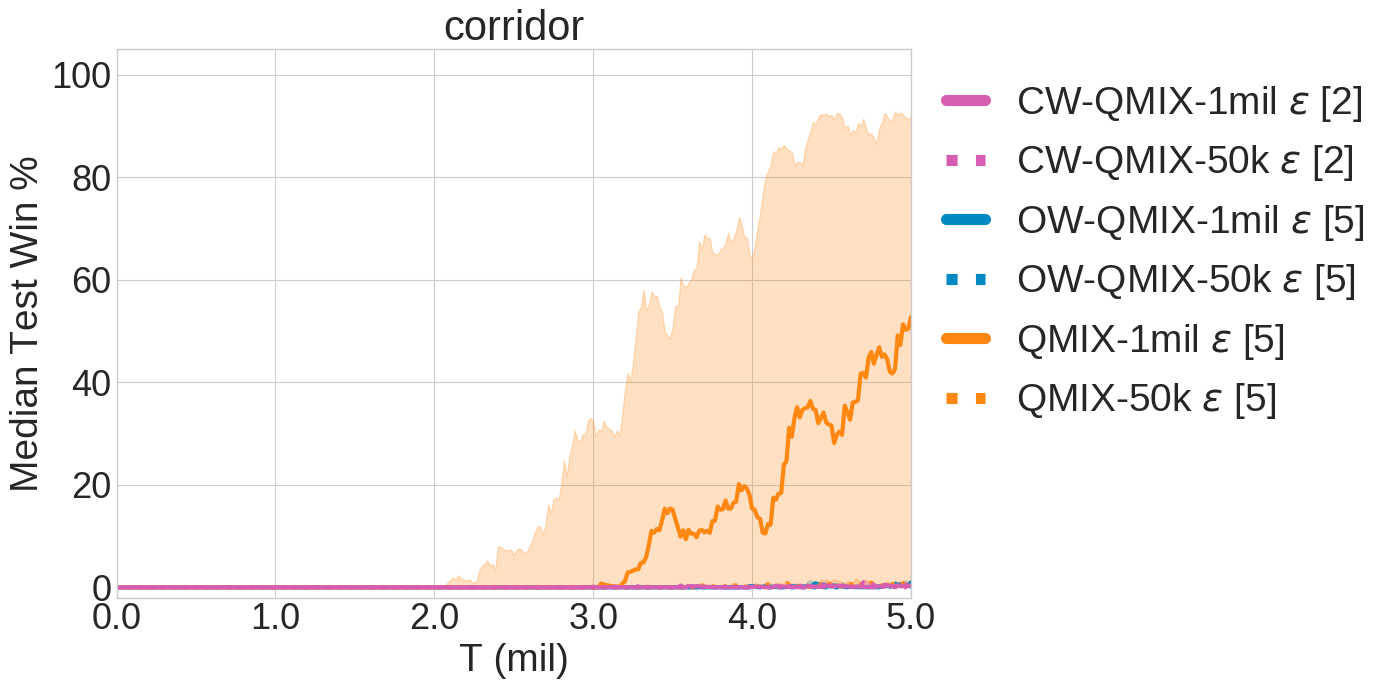}
    \caption{Median test win \% with 2 exploration schedules on 2 \textit{super-hard} SMAC maps.}
    \label{graph:hard_maps}
\end{figure}

Finally, we present results on 2 more \textit{super hard} SMAC maps in order to show the limitations of our method in Figure \ref{graph:hard_maps}.
In \textit{3s5z\_vs\_3s6z} we observe that the extra exploration is not helpful for any method.
Since QMIX is almost able to solve the task, 
this indicates that both exploration and the restricted function class of QMIX are not limiting factors in this scenario.
On \textit{corridor} we see that only QMIX with an extended exploration schedule manages non-zero performance,
showing the importance of sufficient exploration on this map.
The poor performance of Weighted QMIX shows that the extra complexity of our 
method (notably learning \qCentral) can sometimes harm performance, indicating 
that closer attention must be paid to the architecture and weighting functions.

\begin{wrapfigure}{r}{0.4\textwidth}
    \centering
    \includegraphics[width=0.39\columnwidth]{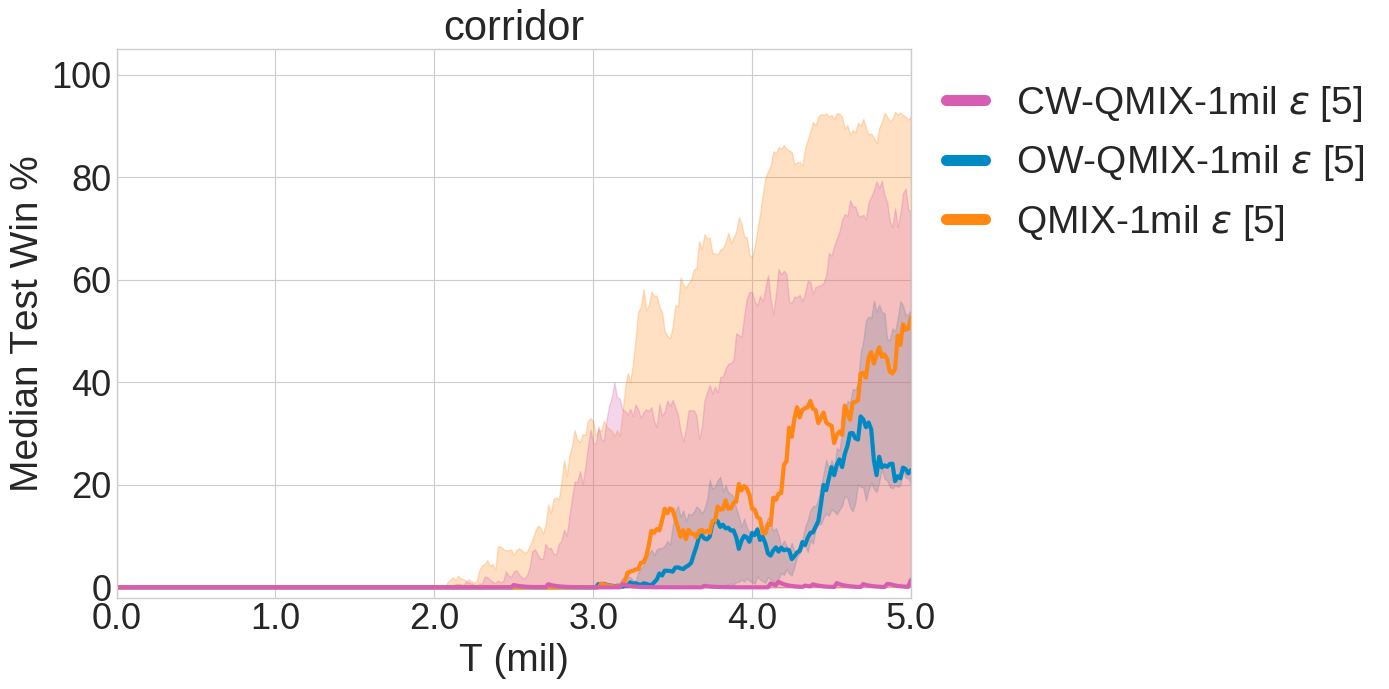}
    \caption{Median test win \% on \textit{corridor} with a modified architecture for \qCentral.}
    \label{graph:corridor_arch}
    \vspace{-0.8cm}
\end{wrapfigure}

Figure \ref{graph:corridor_arch} shows the results on \textit{corridor} for Weighted QMIX with a slightly modified architecture for \qCentral. 
The significantly improved performance for Weighted QMIX indicates that the architecture of \qCentral~is partly responsible for the regression in performance over QMIX.  

\section{Conclusions and Future Work}
\label{sec:conclusion}

This paper presented Weighted QMIX, 
which was inspired by analysing an idealised version of QMIX that first computes the $Q$-learning targets and then projects them into \qmixSpace.
QMIX uses an unweighted projection that places the same emphasis on every joint action, which can lead to suboptimal policies.
Weighted QMIX rectifies this by using a weighted projection that allows more emphasis to be placed on \textit{better} joint actions.
We formally proved that for two specific weightings, the weighted projection is guaranteed to recover the correct maximal joint action for any $Q$.
To fully take advantage of this, we additionally learn an unrestricted joint 
action \qCentral, and prove that it converges to $Q^*$.
We extended Weighted QMIX to deep RL and showed its improved ability to coordinate and its robustness to an increased rate of exploration.
For future work, more complicated weightings could be considered, as opposed to the simplicitic weightings we used where $w$ is either 1 or $\alpha$ in this paper.
Additionally, some of our results demonstrate the limitations of our method, partially stemming from the architecture used for \qCentral.
%!TEX root = main_neurips2020.tex
\section*{Broader Impact} % The NeurIPS latex template suggests us to use unnumbered first level headings for this section.

Due to the broad applicability of cooperative Multi-Agent Reinforcement Learning, we limit our discussion to the cooperative setting in which agents must act independently without communication.
One such potential application is in self-driving cars, in which the agents should be able to make safe and sensible decisions even without access to a communication network. 
Due to the sample inefficiency of current RL methods, combined with their lack of safe exploration it is also necessary to first train agents in a simulated setting before even beginning to train in the real world. 
Hence, the class of algorithms we consider in this paper could be used to train agents in these scenarios and would likely be chosen over fully-decentralised options.
It is important then to obtain a better understanding of the current approaches, in particular of their limitations.
In this paper, we focus primarily on the limitations of QMIX due to its strong performance \citep{samvelyan19smac}.
We also investigate the links between QTRAN and our algorithm and observe poor empirical performance for Actor-Critic style approaches.
Investigating all of these further should improve the performance of all algorithms in this domain, and provide a better understanding of their relative strengths and weaknesses. 

One particular limitation of QMIX is that it can fail in environments in which an agent's best action is dependent on the actions the other agents take, i.e., in environments in which agents must \textit{coordinate} at the same timestep.
However, in a Multi-Agent setting it is often \textit{crucial} to coordinate with the other agents.
Our approach lifts this restriction, and is theoretically able to learn the optimal policy in any environment which greatly increases its applicability.
Extending the capabilities of cooperative MARL algorithms should further extend the applicability of these algorithms in a broader range of applications.
However, our approach introduces extra complexity and can perform poorly in certain challenging domains.
It is important then to consider whether the extra modelling capacity of our method is required to achieve good performance on a selected task.

\section*{Acknowledgments and Disclosure of Funding}

We thank the members of the Whiteson Research Lab for their helpful feedback.
This project has received funding from the European Research
Council (ERC), under the European Union’s Horizon 2020 research and innovation programme
(grant agreement number 637713). It was also supported by an EPSRC grant (EP/M508111/1,
EP/N509711/1) and the UK EPSRC CDT in Autonomous
Intelligent Machines and Systems.
The experiments were made possible by a generous equipment grant from NVIDIA.
Shimon Whiteson is the Head of Research at Waymo, UK.

\bibliography{references}
\bibliographystyle{plainnat}
\onecolumn
\newpage
\appendix

\section{Related Work}
\label{sec:related_work}

In this section we briefly describe related work on cooperative MARL in the common paradigm of
Centralised Training and Decentralised Execution (CTDE).
For a more comprehensive survey of cooperative MARL, see \citep{oroojlooyjadid2019review}.

\citet{tampuu_multiagent_2015} train each agent's DQN \citep{mnih2015human} using Independent $Q$-learning \citep{tan_multi-agent_1993}, which treats the other agents as part of the environment,
which can lead to many pathologies and instability during training \citep{claus1998dynamics,foerster2017stabilising}.
By contrast VDN \citep{sunehag2018value} and QMIX \citep{rashid2018qmix} learn the joint-action $Q$-values, which avoids some of these issues. 
Qatten \citep{yang2020qatten} change the architecture of QMIX's mixing network to a 2-layer linear mixing, in which the weights of the first layer are produced through an attention-based mechanism.
\citet{bohmer2019exploration} learn a centralised joint action $Q$-Value that is approximately maximised by coordinate descent and used to generate trajectories that IQL agents train on.
SMIX$(\lambda)$ \citep{yao_smixlambda:_2019} replaces the 1-step $Q$-learning target with a SARSA$(\lambda)$ target. 
Despite claiming their method can represent a larger class of joint action 
$Q$-values, they can represent exactly the same class as QMIX (\qmixSpace)
since they use the same architecture (in particular, non-negative weights in 
the mixing network). 
\citet{yang2020q} utilise integrated gradients to decompose the joint-action $Q$-values of the critic into individual utilities for each agent, a form of multi-agent credit assignment.
The agents are then regressed against their respective utilities.

\citet{mahajan2019maven} point out some limitations of QMIX arising from its monotonic factorisation.
Specifically, they show that for a specific choice of matrix game, QMIX can fail to learn an optimal policy if each joint action is visited uniformly, which corresponds to our idealised tabular setting in Section \ref{sec:tabular_qmix}.
Additionally, they show that a lower bound on the probability of recovering the optimal policy increases for an $\epsilon$-greedy policy as $\epsilon$ increases.
This is proved by considering the weighting on each joint-action induced by the exploration policy.
By contrast, our weighting is independent of the exploration strategy, adding flexibility.
\wname Weighting uses a smaller weighting when decreasing \qtot~estimates, 
similarly to Hysteretic $Q$-learning \citep{matignon2007hysteretic} which uses 
a smaller learning rate when decreasing value estimates of independent 
learners.

\textbf{Relationship to Actor-Critic.}
Weighted QMIX bears many similarities to an off-policy actor-critic algorithm, 
if we view \qCentral~as the critic and the policy implied by \qtot~as the actor.
Define the deterministic QMIX greedy policy (assuming full observability to simplify the presentation) as: 
\begin{equation*}
\pi_{w}(s) = \begin{pmatrix}
\argmax_{u_1}Q_1(s, u_1)\\
\hdots \\
\argmax_{u_n}Q_n(s, u_n) 
\end{pmatrix}.
\end{equation*}
Weighted QMIX trains \qCentral~to approximate $Q^{\pi_w}$, the $Q$-values of this policy.
This is also an approximation to $Q$-learning since $\pi_{w} \approx \argmax \qCentral$.
Viewed in this manner, Weighted QMIX is similar to MADDPG \citep{lowe2017multi} with a single critic, except for how the actors are trained. MADDPG trains each agent's policy $\pi_a$ via the multi-agent deterministic policy gradient theorem 
, whereas Weighted QMIX trains the policy indirectly by training $Q_{tot}$ via the weighted loss in \eqref{eq:qmix_loss}.
Multi-agent Soft Actor Critic (MASAC) \citep{haarnoja2018soft,iqbal2019actor} is another off-policy actor-critic based approach that instead trains the actors by minimising the KL divergence between each agent's policies and the joint-action $Q$-values.
These actor-critic based approaches (as well as COMA \citep{foerster_counterfactual_2017} and LIIR \citep{du_liir:_2019}) do not restrict the class of joint action $Q$-values they can represent, which theoretically allows them to learn an optimal value function and policy.
However, in practice they do not perform as well as QMIX, perhaps due to 
relative overgeneralisation \citep{wei2018multiagent} or the presence of bad 
local minima.

\textbf{Relationship To QTRAN.}
QTRAN \citep{son2019qtran} is another $Q$-learning based algorithm that 
learns an unrestricted joint action $Q$ function and aims to solve a constrained optimisation problem in order to decentralise it. However, it is empirically hard to scale to more complex tasks (such as SMAC).

We can view QTRAN as specific choices of the 3 components of Weighted QMIX, 
which allows us to better understand its trade-offs and empirical performance 
in relation to WQMIX.
However, the motivations for QTRAN are significantly different.
$\mathbf{\qtot}$ is represented using VDN instead of QMIX, and trained 
using \qCentral~as the target (instead of $y_i$). 
This can limit QTRAN's empirical performance because QMIX generally outperforms 
VDN \citep{samvelyan19smac}.
$\mathbf{\qCentral}$ is a network that takes as input an embedding of 
all agents' chosen actions and observations (and additionally the state if it 
is available). The agent components share parameters with the agent networks 
used to approximate \qtot.

\textbf{Weighting.} The weighting function is as follows:
\begin{equation}
    w(s, \mathbf{u}) = 
    \begin{cases}
    \lambda_{opt} & \mathbf{u} = \mathbf{\hat{u}} \\
    \lambda_{nopt} & \qtot(s,\mathbf{u}) < \qCentral(s,\mathbf{u}) \\
    0 & \text{otherwise},
    \end{cases}
    \label{eq:qtran_w}
\end{equation}
where $\mathbf{\hat{u}} = \argmax Q_{tot}(s,\cdot)$.
Using too small a weight in the weighting can have a substantial negative effect on performance (as we show in Appendix \ref{app:results}).
However, using a 0 weight for overestimated $Q$-values is a fundamental part of the QTRAN algorithm.

\textbf{Concurrent Work.} 
\citet{wang2020qplex} propose QPLEX, which also expands the class of joint-action $Q$-values that can be represented. 
They acheive this by decomposing $Q_{tot}$ as a sum of a value function and a non-positive advantage function. 
Crucially, this advantage function is $0$ for the joint-action in which every agent maximises their own utilities ($\hat{\mathbf{u}}^*$ in our notation).
This ensures consistency between the agent's greedy joint-action ($\hat{\mathbf{u}}^*$), and the true maximum joint-action of $Q_{tot}$.
In contrast, Weighted QMIX does not maintain this consistency.
Our experimental results show that despite not restricting the class of joint-action $Q$-values that can be represented, QPLEX can struggle to learn a good policy in some environments like QMIX.

\citet{son2020qopt} propose QOPT, which also learns an unrestricted $\qCentral$ and utilises an optimistic-style weighting to train $Q_{tot}$ that is represented by QMIX.
In contrast to Weighted QMIX, whose $\qCentral$ does not share any parameters with $Q_{tot}$, the unrestricted joint-action $Q$-values of QOPT are obtained through an unrestricted mixing network (the weights are not constrained to be non-negative) which takes as input the agent utilities.
The weighting function used for training $Q_{tot}$ is similar to the weighting for CW-QMIX, in which a smaller weighting is used for the the joint-actions whose $Q$-values are estimated as lower than the $Q$-values for the approximate best joint-action ($\hat{\mathbf{u}}^*$).   

\section{Proof of Theorems}
\label{app:proofs}

% Propositions
\begin{proposition}
\label{prop:wargmax}
For any $w: S \times \mathbf{U} \rightarrow (0,1]$ 
and $Q$.
Let $Q_{tot} = \wProj Q$. 
Then $\forall s \in S$, $\approxU \in \argmax Q_{tot}(s,\cdot)$.
We have that $Q_{tot}(s, \approxU) \geq Q(s, \approxU)$.
If $\approxU = \U^* := \argmax_{\U} Q(s,\U)$ then $Q_{tot}(s,\approxU) = Q(s, \approxU)$.
\end{proposition}
\vspace{-0.3cm}
\begin{proof}
Consider a $s \in S$.
Assume for a contradiction that $Q_{tot}(s, \approxU) < Q(s, \approxU)$.

Define $Q'_{tot}$ as follows:
\begin{equation*}
    Q'_{tot}(s, \U) = 
    \begin{cases}
    Q(s,\U) & \U = \approxU \\
    Q_{tot}(s,\U) & \text{otherwise},
    \end{cases}
\end{equation*}

By construction we have that $Q'_{tot} \in \qmixSpace$, and

\begin{align*}
&\sum_{\U \in \mathbf{U}} w(s,\U)(Q(s,\U) - Q'_{tot}(s,\mathbf{u}))^2\\ 
&= \sum_{\U \neq \approxU} w(s,\U)(Q(s,\U) - Q'_{tot}(s,\mathbf{u}))^2
 + w(s,\approxU)(Q(s,\approxU) - Q'_{tot}(s,\approxU))^2\\
 &= \sum_{\U \neq \approxU} w(s,\U)(Q(s,\U) - Q'_{tot}(s,\mathbf{u}))^2\\ 
 & \qquad (Q'_{tot}(s,\approxU) = Q(s, \approxU))\\
 &= \sum_{\U \neq \approxU} w(s,\U)(Q(s,\U) - Q_{tot}(s,\mathbf{u}))^2\\
 & \qquad (Q'_{tot}(s,\U) = Q_{tot}(s, \U)~\forall \U \neq \approxU)\\
 &< \sum_{\U \neq \approxU} w(s,\U)(Q(s,\U) - Q_{tot}(s,\mathbf{u}))^2 + w(s,\approxU)(Q(s,\approxU) - Q_{tot}(s,\approxU))^2\\
 & \qquad (Q_{tot}(s,\approxU) < Q(s,\approxU)) \\
 &= \sum_{\U \in \mathbf{U}} w(s,\U)(Q(s,\U) - Q_{tot}(s,\mathbf{u}))^2.
\end{align*}

Thus $Q_{tot}$ cannot be the solution of \wProj $Q$, a contradiction.
And so $Q_{tot}(s, \approxU) \geq Q(s, \approxU)$. 

Now consider the scenario in which $\approxU = \U^*$, and
assume for a contradiction that $Q_{tot}(s, \approxU) > Q(s, \approxU)$.

Define $Q'_{tot}$ as follows:
\begin{equation*}
    Q'_{tot}(s, \U) = 
    \begin{cases}
    Q(s,\U) & \U = \approxU = \U^* \\
    \min\{Q_{tot}(s,\U), Q(s,\U^*)\} & \text{otherwise},
    \end{cases}
\end{equation*}

Again, by construction $Q'_{tot} \in \qmixSpace$. 

\begin{align*}
&\sum_{\U \in \mathbf{U}} w(s,\U)(Q(s,\U) - Q'_{tot}(s,\mathbf{u}))^2\\
&= \sum_{\U \neq \approxU} w(s,\U)(Q(s,\U) - Q'_{tot}(s,\mathbf{u}))^2
 + w(s,\approxU)(Q(s,\approxU) - Q'_{tot}(s,\approxU))^2\\
 &= \sum_{\U \neq \approxU} w(s,\U)(Q(s,\U) - Q'_{tot}(s,\mathbf{u}))^2 \\
 & \qquad (Q'_{tot}(s,\approxU) = Q(s, \approxU))\\
 &\leq \sum_{\U \neq \approxU} w(s,\U)(Q(s,\U) - Q_{tot}(s,\mathbf{u}))^2 \\
 & \qquad (\text{If $\min\{Q_{tot}(s,\U), Q(s,\approxU)\}=Q_{tot}(s,\U)$ then $Q'_{tot}(s,\U)=Q_{tot}(s,\U)$.}\\
 & \qquad \text{Otherwise $Q(s,\U^*) < Q_{tot}(s,\U) \implies (Q(s,\U^*) - Q(s, \U))^2 < (Q_{tot}(s,\U) - Q(s, \U))^2$},\\
 & \qquad \text{since $Q(s,\U) \leq Q(s,\U^*)$.})\\
 &< \sum_{\U \neq \approxU} w(s,\U)(Q(s,\U) - Q_{tot}(s,\mathbf{u}))^2 + w(s,\approxU)(Q(s,\approxU) - Q_{tot}(s,\approxU))^2 \\
 & \qquad (Q_{tot}(s,\approxU) > Q(s,\approxU)) \\
 &= \sum_{\U \in \mathbf{U}} w(s,\U)(Q(s,\U) - Q_{tot}(s,\mathbf{u}))^2.
\end{align*}

Thus, $Q_{tot}$ cannot be the solution of \wProj $Q$, a contradiction. 
This proves that $Q_{tot}(s, \approxU) = Q(s, \approxU)$ if $\approxU = \U^*$.
\end{proof}

\begin{proposition}
\label{prop:exists_argmax}
Let $Q_{tot} = \wProj Q$.
$\forall s \in S$
$\exists \approxU \in \argmax Q_{tot}(s,\cdot)$ such that $Q_{tot}(s,\approxU) = Q(s,\approxU)$.
\end{proposition}

\begin{proof}
Assume for a contradiction that $\forall \approxU \in \argmax Q_{tot}$ we have that $Q_{tot}(\approxU) > Q(\approxU)$.

Define $\Delta_s := Q_{tot}(s,\approxU) - \max \{Q_{tot}(s,\U) | \U \in \mathbf{U}, Q_{tot}(s,\U) < Q_{tot}(s,\approxU)\}$
to be the difference between the maximum $Q$-Value and the next biggest $Q$-Value (the action gap \citep{bellemare2016increasing}).
$\Delta_s$ is well defined as long as there exists a sub-optimal action.
If there is not a suboptimal action, then trivially any $\U \in \mathbf{U}$ satisfies the condition.

Let $\epsilon = \min\{\Delta_s / 2, (Q_{tot}(s,\approxU) - \max\{Q(s,\U)|\U \in \argmax Q_{tot}(s,\cdot)\})/2\} > 0$.

Define $Q'_{tot}$ as follows:
\begin{equation*}
    Q'_{tot}(s, \U) = 
    \begin{cases}
    Q_{tot}(s, \U) - \epsilon & \U \in \argmax Q_{tot} \\
    Q_{tot}(s, \U) & \text{otherwise}.
    \end{cases}
\end{equation*}
i.e. we have decreased the $Q$-Value estimates for the argmax joint actions by a small non-zero amount. Since $\epsilon < \Delta_s$ we do not need to worry about adjusting other action's estimates.

By construction $Q'_{tot} \in \qmixSpace$. 

Then $Q'_{tot}$ has a smaller loss than $Q_{tot}$ since the estimates for the argmax actions are closer to the true values. 

This gives our contradiction since $Q_{tot} \in \wProj Q$.

Thus $\exists \approxU \in \argmax Q_{tot}$ such that $Q_{tot}(\approxU) \leq Q(\approxU)$.
Combined with Proposition \ref{prop:wargmax} gives us the required result.
\end{proof}

\begin{corollary}
\label{cor:unique_argmax_match}
If $Q_{tot}$ has a unique argmax \approxU, then $Q_{tot}(\approxU) = Q(\approxU)$.
\end{corollary}

\begin{proof}
Proposition \ref{prop:exists_argmax} showed the existence of an argmax action whose $Q_{tot}$-value matches $Q$ exactly. If there is a unique argmax \approxU, then it must match exactly giving our result.
\end{proof} 

% Theorems
\addtocounter{theorem}{-1}
\addtocounter{theorem}{-1}
\begin{theorem}
Let $w$ be the Idealised Central Weighting from Equation \eqref{eq:critic_w}.
Then $\exists \alpha>0$ such that, $\argmax \Pi_w Q = \argmax Q$ for any $Q$.
\end{theorem}
\begin{theorem}
Let $w$ be the \wname Weighting from Equation \eqref{eq:hyst_w}.
Then $\exists \alpha>0$ such that, $\argmax \Pi_w Q = \argmax Q$ for any $Q$.
\end{theorem}
\begin{proof}

Since the proof of both theorems contains a significant overlap, we will merge them both into a single proof.

We will start by first considering the Idealised Central Weighting: 
Let $Q_{tot} = \Pi_w Q$ ($Q_{tot} \in \Pi_w Q$ if there are distinct solutions).

Let $\U^* \in \argmax Q$, be an optimal action.

Consider a state $s \in S$.

Define $\Delta_s := Q(s,\U^*) - \max \{Q(s,\U) | \U \in \mathbf{U}, Q(s,\U) < Q(s,\U^*)\}$
to be the difference between the maximum $Q$-Value and the next biggest $Q$-Value (the action gap \citep{bellemare2016increasing}).
$\Delta_s$ is well defined as long as there exists a sub-optimal action.
If there is not a suboptimal action, then trivially $\argmax \wProj Q = \argmax Q$ for state $s$. 

Let $\approxU \in \argmax \wProj Q$, and consider the loss when $\mathbf{\hat{u}} = \mathbf{u}^*$.

By Propositon \ref{prop:wargmax} we have that $Q_{tot}(s,\approxU) = Q(s,\approxU)$.

Then the loss:
$$
\sum_{\mathbf{u} \in \mathbf{U}} w(s,\mathbf{u})(Q(s,\mathbf{u}) - Q_{tot}(s,\mathbf{u}))^2 = 
\alpha \sum_{\mathbf{u} \neq \mathbf{u}^*} (Q(s,\mathbf{u}) - Q_{tot}(s,\mathbf{u}))^2 < \alpha (\frac{R_{max}}{1-\gamma})^2 |U|^n,
$$
where $R_{max} := \max r - \min r$.
The last inequaility follows since the maximum difference between $Q$-values in the discounted setting is then $\frac{R_{max}}{1-\gamma}$, and there are $|U|^n$ joint-actions total. 

Whereas if $\approxU \neq \U^*$, then the loss
\begin{align*}
&\sum_{\mathbf{u} \in \mathbf{U}} w(s,\mathbf{u})(Q(s,\mathbf{u}) - Q_{tot}(s,\mathbf{u}))^2\\
&= (Q(s,\mathbf{u}^*) - Q_{tot}(s, \mathbf{u}^*))^2 + 
~\alpha \sum_{\mathbf{u} \neq \mathbf{u}^*} (Q(s,\mathbf{u}) - Q_{tot}(s,\mathbf{u}))^2 \\
&\geq \Delta_s^2,
\end{align*}
since $Q(s,\mathbf{u}^*) - Q_{tot}(s, \mathbf{u}^*) \geq \Delta_s$, which is proved below.

By Proposition \ref{prop:exists_argmax} let $\approxU' \in \argmax Q_{tot}$ such that $Q_{tot}(s,\approxU') = Q(s,\approxU')$.
Then $Q(s,\U^*) \geq \Delta_s + Q(s, \approxU') = \Delta_s + Q_{tot}(s,\approxU') > \Delta_s + Q_{tot}(s, \mathbf{u}^*) \implies Q(s,\mathbf{u}^*) - Q_{tot}(s, \mathbf{u}^*) > \Delta_s$.\\
The strict inequaility $Q_{tot}(s,\approxU') > Q_{tot}(s, \mathbf{u}^*)$ used is due to $\approxU \neq \U^*$.

Setting $0 < \alpha_s < \frac{\Delta_s^2 (1-\gamma)^2}{(R_{max})^2 |U|^n}$ then gives the required result for state $s$.

Letting $\alpha = \min_s \alpha_s > 0$ completes the proof for the Idealised Central Weighting.

For the proof of the \wname Weighting we will use many of the same arguments and notation.

We will once again consider a single state $s \in S$ and the action gap $\Delta_s$.

Now, let us consider a \qtot~of a specific form:
For $s \in S$, let $Q_{tot}(s, \mathbf{\hat{u}}) = c_s + \epsilon,$ where $\epsilon << \Delta_s$ and $Q_{tot}(s, \mathbf{u}) = c_s, \forall \mathbf{u} \neq \mathbf{\hat{u}}$. Note that here \qtot~has a unique maximum action.

For this $Q_{tot}$, consider $\mathbf{\hat{u}} = \mathbf{u}^*$ then the loss:
$$
\sum_{\mathbf{u} \in \mathbf{U}} w(s,\mathbf{u})(Q(s,\mathbf{u}) - Q_{tot}(s,\mathbf{u}))^2 = 
\alpha \sum_{\mathbf{u} \neq \mathbf{u}^*} (Q(s,\mathbf{u}) - Q_{tot}(s,\mathbf{u}))^2 \leq \alpha (\frac{R_{max}}{1-\gamma})^2 |U|^n < \Delta_s^2,
$$
since $Q_{tot}(s,\mathbf{u}) = c_s = Q(s,\mathbf{u}^*) - \epsilon > Q(s, \mathbf{u}), \forall \mathbf{u} \neq \mathbf{u}^*$ by Proposition \ref{prop:wargmax}, which means that $w(s,\mathbf{u} \neq \mathbf{u}^*)=\alpha$. The final inequality follows due to setting $0 < \alpha_s < \frac{\Delta_s^2 (1-\gamma)^2}{(R_{max})^2 |U|^n}$ as earlier.

Now consider any $\qtot' \in \qmixSpace$. 

If for this $\qtot'$, $\approxU \neq \mathbf{u}^*$, then $w(s,\U^*) = 1$ and thus the loss
\begin{align*}
&\sum_{\mathbf{u} \in \mathbf{U}} w(s,\mathbf{u})(Q(s,\mathbf{u}) - Q_{tot}(s,\mathbf{u}))^2 \\
&= (Q(s,\mathbf{u}^*) - Q_{tot}(s, \mathbf{u}^*))^2 + 
\sum_{\mathbf{u} \neq \mathbf{u}^*} w(s,\mathbf{u}) (Q(s,\mathbf{u}) - Q_{tot}(s,\mathbf{u}))^2 \\
&\geq \Delta_s^2.
\end{align*}

By Proposition \ref{prop:exists_argmax} let $\approxU' \in \argmax Q'_{tot}$ such that $Q'_{tot}(s,\approxU') = Q(s,\approxU')$.
Since $Q(s,\mathbf{u}^*) \geq \Delta_s + Q(s, \approxU') = \Delta_s + 
\qtot(s,\approxU') > \Delta_s + Q_{tot}(s, \mathbf{u}^*) \implies Q(s,\mathbf{u}^*) - Q_{tot}(s, \mathbf{u}^*) > \Delta_s$. 

Thus, we have shown that for any $\qtot'$ with $\approxU \neq \mathbf{u}^*$ the loss is greater than the $Q_{tot}$ we first considered with $\mathbf{\hat{u}} = \mathbf{u}^*$.

And so for state $s$, $\argmax \Pi_w Q(s,\cdot) = \mathbf{u}^* = \argmax Q(s,\cdot)$.

Letting $\alpha = \min_s \alpha_s > 0$ once again completes the proof.
\end{proof}

\begin{corollary}
Letting $w$ be the Central or \wname Weighting, then
$\exists \alpha > 0$ such that
the unique fixed point of \wOper~ is $Q^*$.
Furthermore, $\wProj Q^* \subseteq \qmixSpace$ recovers an optimal policy, and
$\max \wProj Q^* = \max Q^*$.
\end{corollary}

\begin{proof}

Using the results of Theorems \ref{th:critic_w} and \ref{th:hyst_w} we know that $\exists \alpha > 0$ such that $\argmax \wProj Q = \argmax Q$.
We also know from their proofs that the same $\alpha$ works for both weightings.

Instead of updating \qtot~in tandem with \qCentral, we can instead write \wOper~ as:
$$
\wOper \qCentral (s,\mathbf{u}) := \mathbb{E} [r + \gamma \qCentral(s',\argmax_{\U'} (\wProj \qCentral)(s',\U'))].
$$

And so:
$$
\qCentral(s',\argmax_{\U'} (\wProj \qCentral)(s',\U')) = \max_{\U'} \qCentral(s',\U'),~\forall s' \in S.
$$

Thus, our operator \wOper~ is equivalent to the usual Bellman Optimality Operator \optimOper, which is known to have a unique fixed point $Q^*$ \citep{watkins1992q,melo2001convergence}.

Once again by the results of Theorems \ref{th:critic_w} and \ref{th:hyst_w}, we know that $Q^*_{tot} \in \wProj Q^*$ acheives the correct argmax for every state.
Thus it is an optimal policy.
Finally, Proposition \ref{prop:wargmax} shows that $\max \wProj Q^* = \max Q^*$.

\end{proof}

\section{Experimental Setup}
\label{sec:setup}

We adopt the same training setup as \citep{samvelyan19smac}, using PyMARL to run all experiments.
The architecture for QMIX is also the same as in \citep{samvelyan19smac}.

The architecture of the mixing network for \qCentral is a feed forward network with 3 hidden layers of $256$ dim and ReLU non-linearities. Shown in Figure \ref{fig:wqmix_targets}.

For the experiment in Figure \ref{graph:corridor_arch} the architecture for \qCentral is modified slightly.
We replace the first hidden layer with a hypernetwork layer. 
A hypernetwork with a single hidden layer of dim $64$ (with ReLU) takes the state as input and generates the weight matrix.
Inspired by \citep{yang2020qatten} we then take the column-wise softmax of this weight matrix, which can be viewed as an approximation of multi-head attention.

The architecture and setup for QTRAN is also the same, except we use a 3 layer feedforward network of dim $\{64, 256\}$ to match the depth of \qCentral.

MADDPG and MASAC's critic shares the same architecture as \qCentral.

MADDPG is trained using the deterministic multi-agent policy-gradient theorem, via the Gumbel-Softmax trick, as in \citep{lowe2017multi,iqbal2019actor}.
Specifically for agent $a$ we produce $Q(s,u_a, u^{-a})$ (with target network $Q$) for each possible action, where $u^{-a}$ are the actions of the other agents produced by their most recent policies.
We multiply these by the agent's policy (one-hot vector since it is deterministic) and use the Straight Through Gumbel-Softmax estimator \citep{jang2016categorical} to differentiate through this with a temperature of $1$.  

MASAC is trained by minimising the KL divergence between each agent's policy $\pi_a$ and $\exp(Q(s,\cdot, u^{-a}) - \alpha_{ent} \log \pi_a)$.,
Since the KL divergence is an expectation: $\mathbb{E}_{\pi_a}[\log(\frac{\pi_a}{\exp(Q(s,\cdot, u^{-a}) - \alpha_{ent} \log \pi_a)}]$, we approximate it by sampling an action from $\pi_a$ for each agent. These sampled actions are used for $u^{-a}$.
For the actor's policies we use the same $\epsilon$-greedy floor technique as in \citep{foerster_counterfactual_2017}.

QPLEX uses the same setup for its mixing network as for the SMAC experiments in \citep{wang2020qplex}.

\subsection{Predator Prey}

For Weighted QMIX variants (and ablations with just a weighting),
we consider $\alpha \in \{0.1, 0.5\}$ and set $\alpha=0.1$ for all variants.

For QTRAN we set $\lambda_{opt}=1$ and consider $\lambda_{nopt} \in \{0.1, 1, 10\}$ (since only their relative weighting makes a difference), and the dim of the mixing network in $\{64, 256\}$.
We set $\lambda_{nopt}=10$ and the dim of the mixing network to $64$.

For MASAC we consider $\alpha_{ent} \in \{0, 0.001, 0.01, 0.1\}$ and set it to $0.001$.

\subsection{SMAC Robustness to exploration}

For Weighted QMIX we consider $\alpha \in \{0.01, 0.1, 0.5, 0.75\}$ and set $\alpha=0.75$ for CW and $\alpha=0.5$ for OW. All lines are available in Appendix \ref{app:results}.

For the Weighted QMIX ablations we considered $\alpha \in \{0.5, 0.75\}$ and set $\alpha=0.75$.

For QTRAN we set $\lambda_{opt}=1$ and consider $\lambda_{nopt} \in \{1, 10\}$ (since these 2 performed best in preliminary experiments), and the dim of the mixing network in $\{64, 256\}$.
We set $\lambda_{nopt}=10$ for \textit{5m\_vs\_6m} and $\lambda{nopt}=1$ for \textit{3s5z}. The dim of the mixing network is set to $64$.

For MASAC we consider $\alpha_{ent} \in \{0, 0.001, 0.01\}$ and set it to $0$ for \textit{3s5z} and $0.01$ for \textit{5m\_vs\_6m}. 

\subsection{SMAC Super Hard Maps}

We consider $\alpha \in \{0.01, 0.1, 0.5, 0.75\}$ and set $\alpha=0.5$ for OW-QMIX and $\alpha=0.75$ for CW-QMIX.

For the experiment in Figure \ref{graph:corridor_arch} we only considered $\alpha \in \{0.5, 0.75\}$ and set $\alpha=0.75$ for both methods.

\section{Ablations}

In order to better understand our method, we examine 3 additional baselines:

\textbf{QMIX + \qCentral.} This ablation removes the weighting in our loss ($w = 1$), but still uses \qCentral to bootstrap. This ablation allows us to see if a better bootstrap estimate alone can explain the performance of WQMIX.

\textbf{QMIX + CW/OW.} This ablation introduces the CW/OW weightings into QMIX's loss function for \qtot. We do not additionally learn \qCentral.

\section{Results}
\label{app:results}

In this section we present the results of additional experiments that did not fit in the main paper.

\textbf{Ablation experiments for Predator Prey.}
\begin{figure}
    \centering
    \includegraphics[width=0.49\columnwidth]{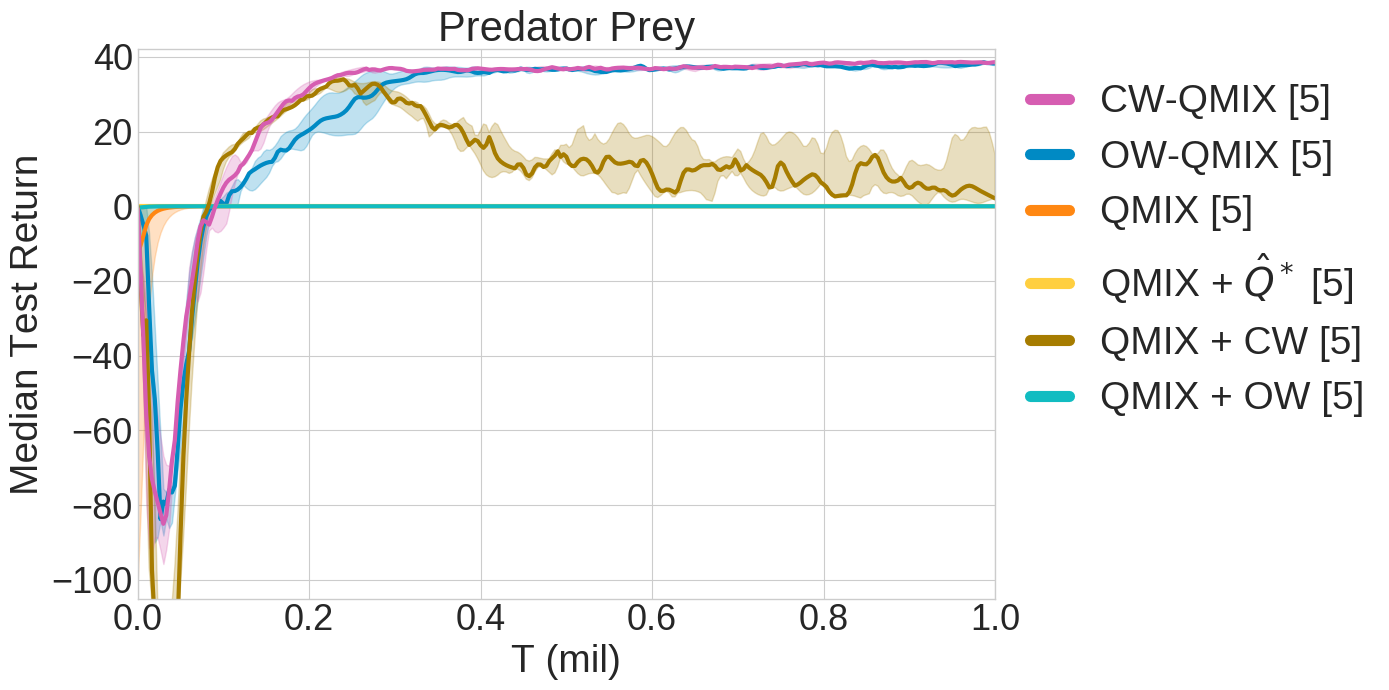}
    \caption{Median test return for the predator prey environment, comparing Weighted QMIX and 3 ablations.}
    \label{graph:pred_prey_abl}
\end{figure}

Our ablation experiments in Figure \ref{graph:pred_prey_abl} show the necessity for both \qCentral and a weighting in order to solve this task.
As expected QMIX $+ \qCentral$ is unable to solve this task due to the challenges of relative overgeneralisation. 
The use of a uniform weighting in the projection prevents the learning of an optimal policy in which two agents coordinate to capture a prey.
Thus, even if \qCentral can theoretically represent the $Q$-values of this optimal policy, the QMIX agents are unable to recover it.
Figure \ref{graph:pred_prey_abl} also shows that QMIX with just a weighting in its projection (and no \qCentral) is unable to successfully solve the task. 

\textbf{Ablation experiments testing robustness to increased exploration.}

\begin{figure*}
    \includegraphics[width=0.49\columnwidth]{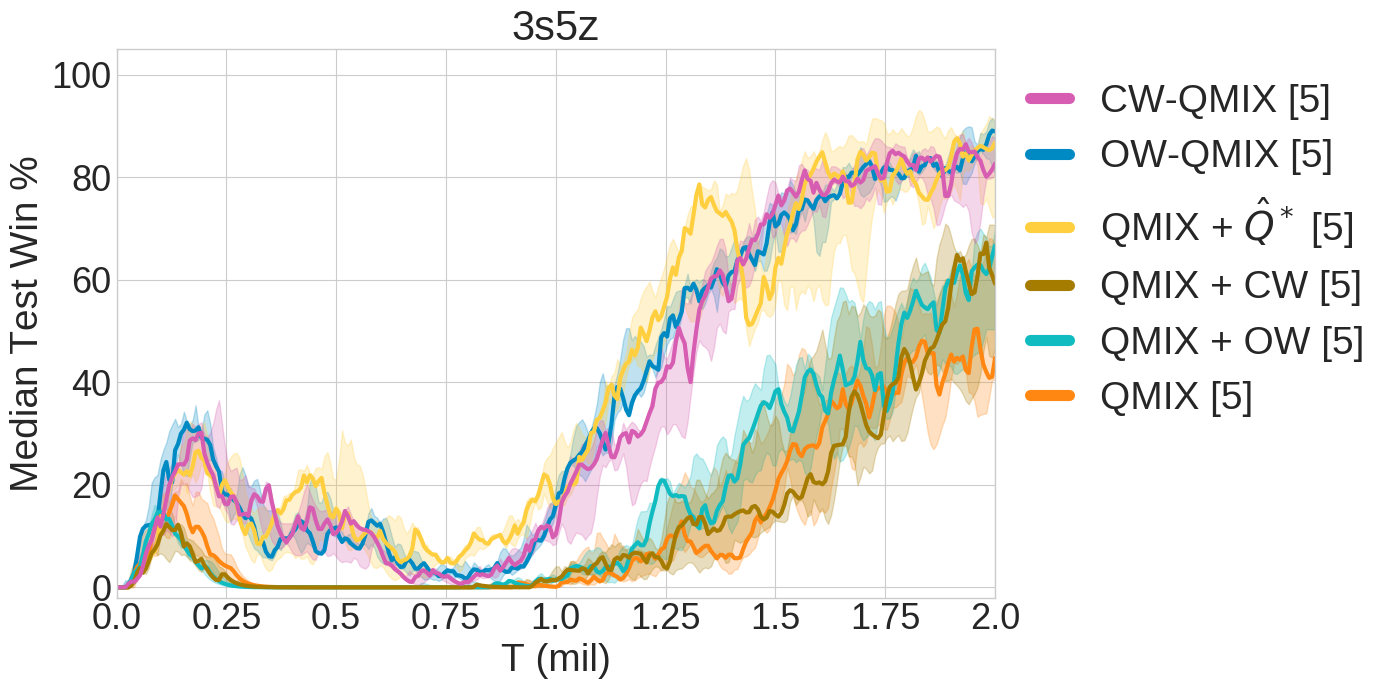}
    \includegraphics[width=0.49\columnwidth]{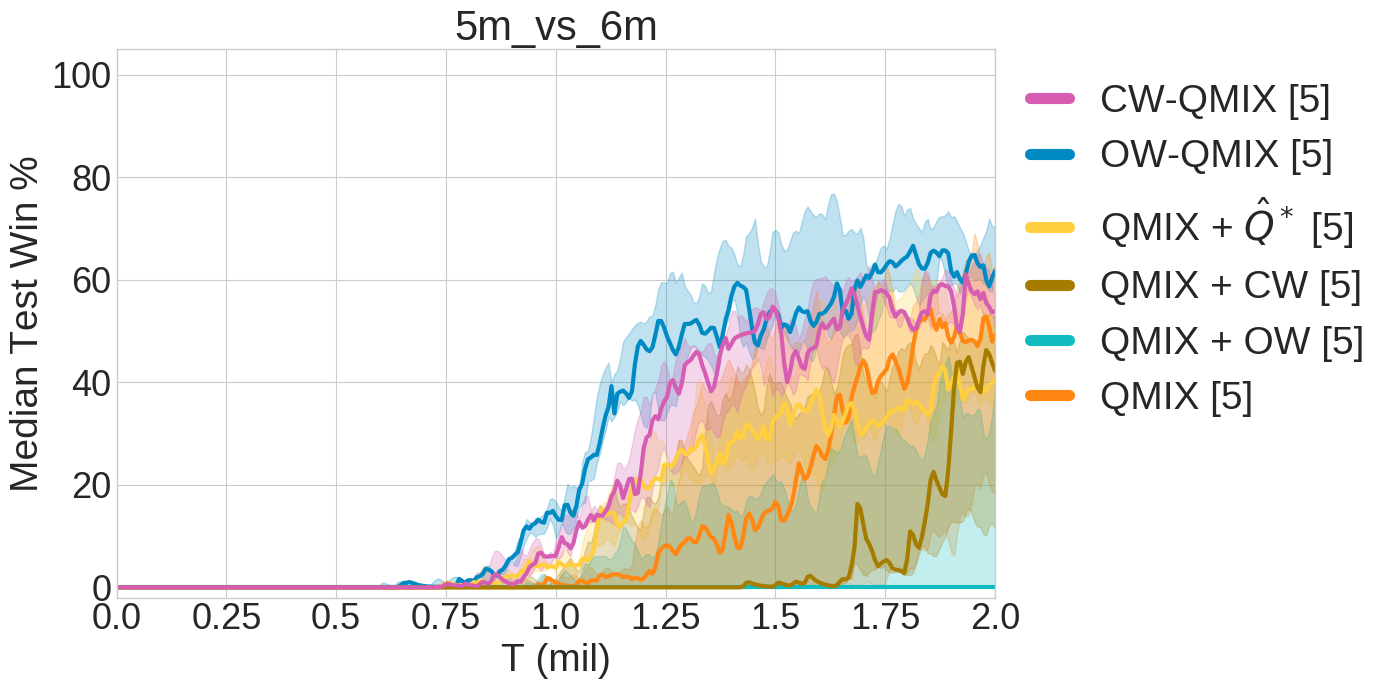}
    \caption{Median test win \% with an increased rate of exploration, comparing Weighted QMIX and 3 ablations.}
    \label{graph:1mil_eps_smac_abl}
\end{figure*}

Figure \ref{graph:1mil_eps_smac_abl} shows the results of further ablation experiments, confirming the need for both \qCentral and a weighting to ensure consistent performance.
Note in particular that combining QMIX with a weighting results in significantly worse performance in \textit{5m\_vs\_6m}, and no better performance in \textit{3s5z}.

\textbf{Effect of $\alpha$ on performance.}

\begin{figure}
    \includegraphics[width=0.49\columnwidth]{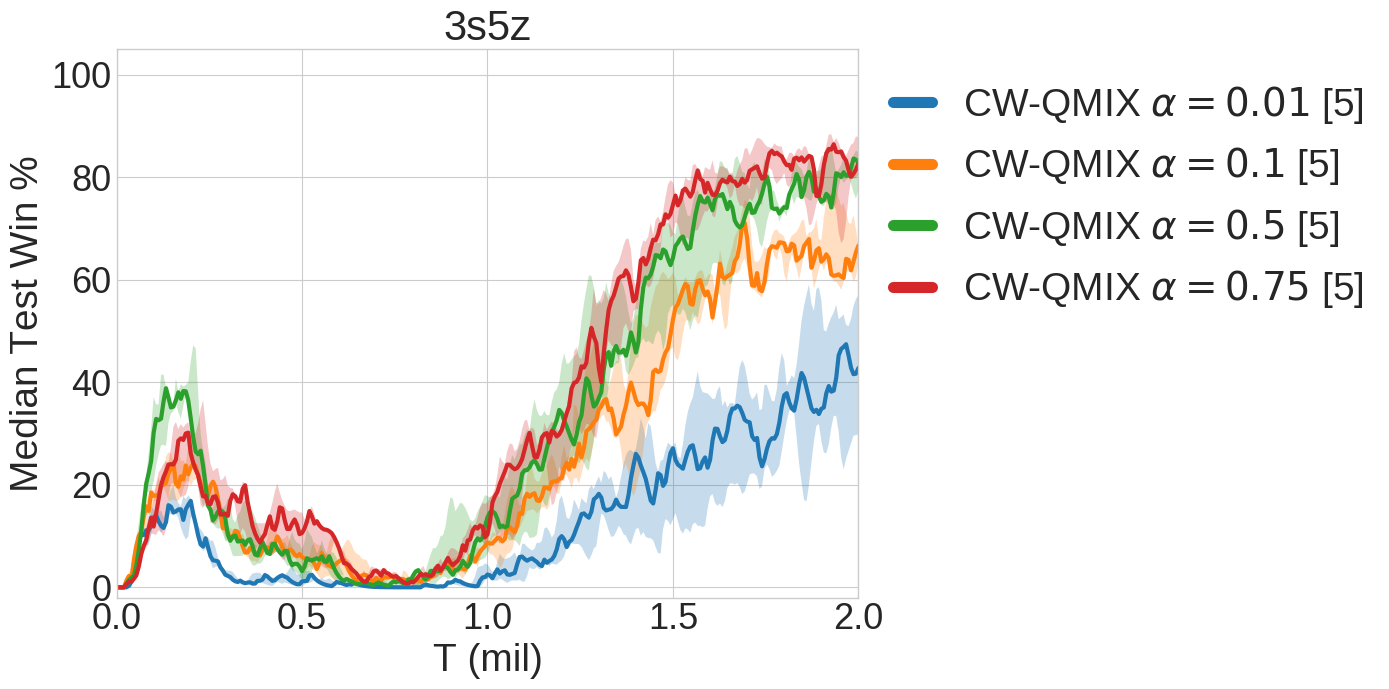}
    \includegraphics[width=0.49\columnwidth]{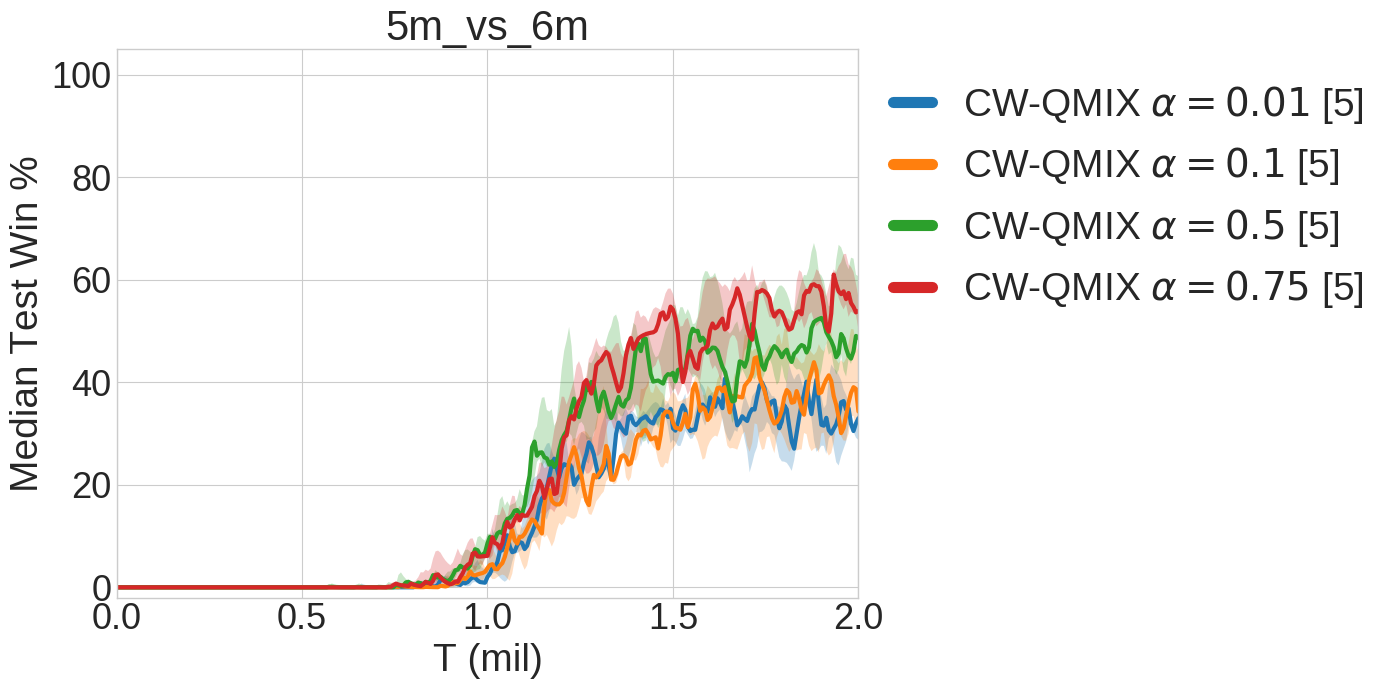}
    \includegraphics[width=0.49\columnwidth]{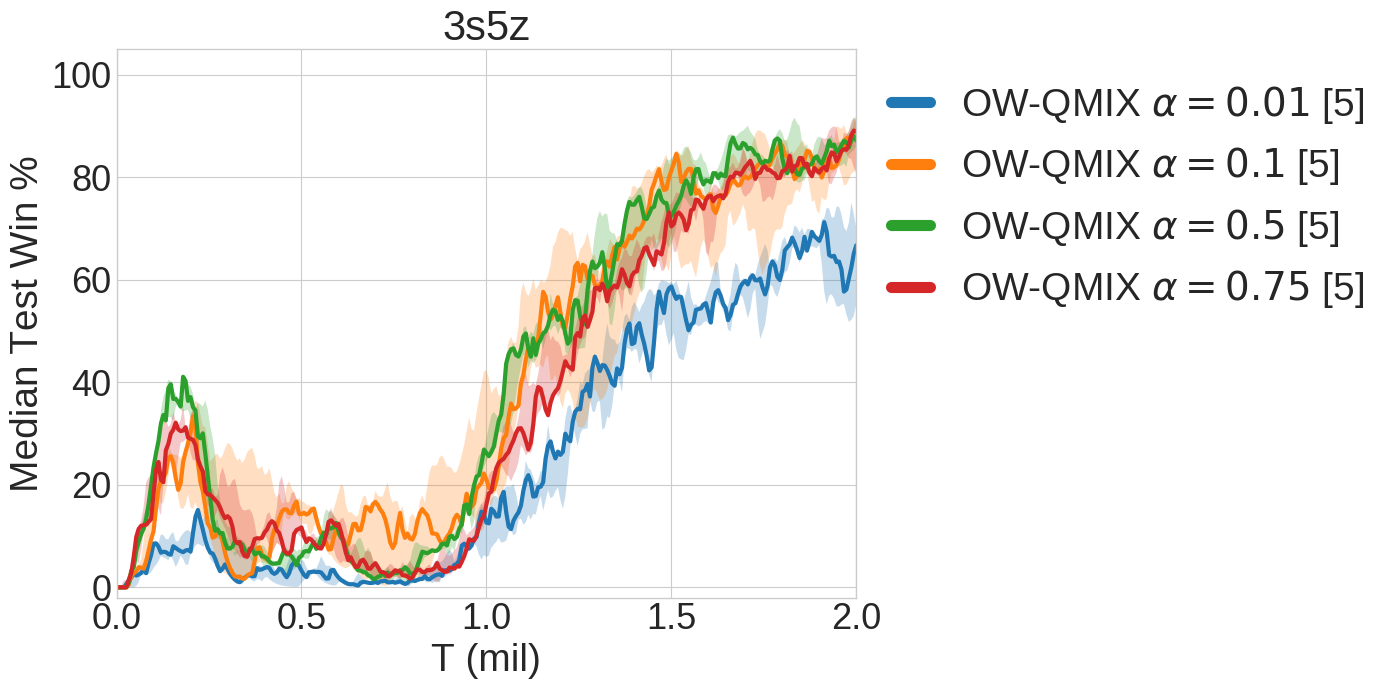}
    \includegraphics[width=0.49\columnwidth]{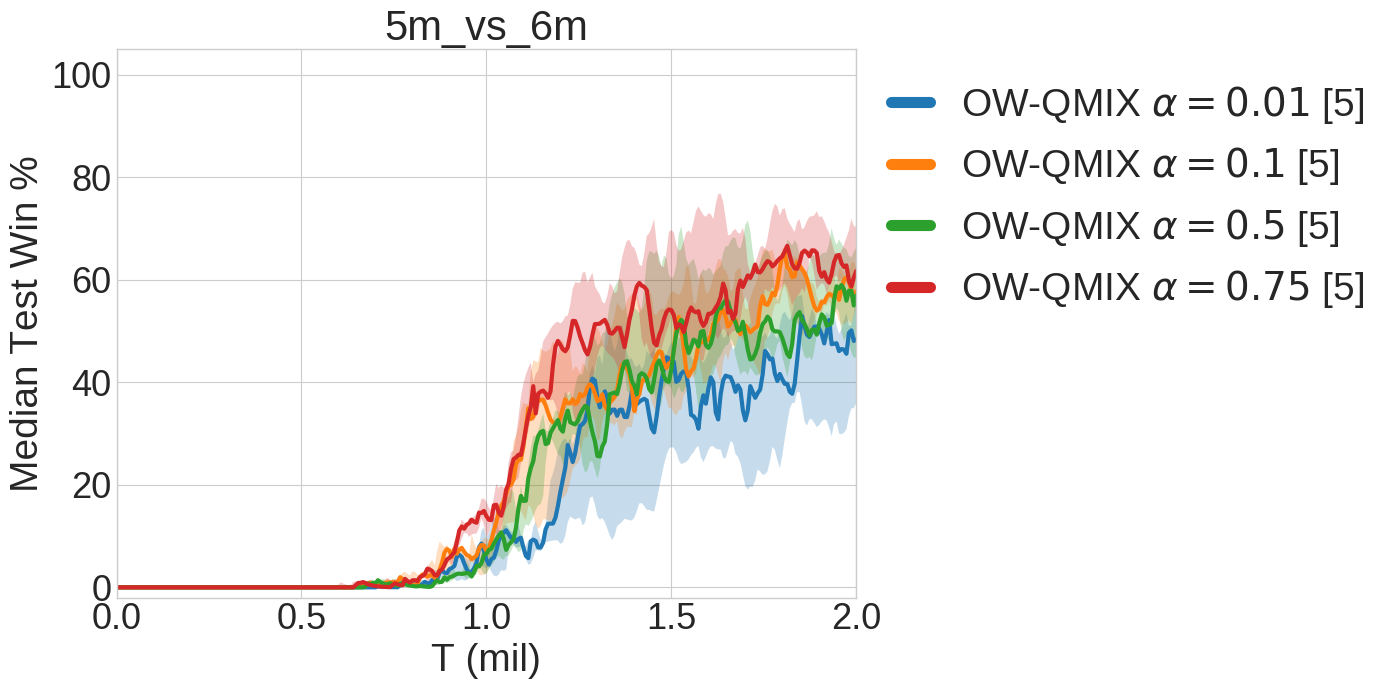}

    \caption{Median test win \% with an increased rate of exploration. \textbf{Above:} The effect of varying $\alpha$ for CW-QMIX. \textbf{Below:} The effect for OW-QMIX.}
    \label{graph:1mil_eps_w_exps}
\end{figure}

Figure \ref{graph:1mil_eps_w_exps} shows the effect of varying $\alpha$ in the weighting function for CW-QMIX and OW-QMIX.
We can see that if $\alpha$ is too low, performance degrades considerably.

\end{document}